\documentclass[letter,11pt]{article}

\usepackage[utf8]{inputenc} % allow utf-8 input
\usepackage[T1]{fontenc}    % use 8-bit T1 fonts
\usepackage[colorlinks, allcolors=blue]{hyperref}       % hyperlinks
\usepackage{url}            % simple URL typesetting
\usepackage{booktabs}       % professional-quality tables
\usepackage{amsfonts}       % blackboard math symbols
\usepackage{nicefrac}       % compact symbols for 1/2, etc.
\usepackage{microtype}      % microtypography
\usepackage[margin = 1in]{geometry}

\usepackage{amsmath}
\usepackage{amsthm}
\usepackage{amssymb}
\usepackage{bbm}
\usepackage{enumitem}
\usepackage{textcomp}
\usepackage{natbib}

\usepackage{mathtools}
\DeclarePairedDelimiter{\norm}{\lVert}{\rVert}

\newtheorem{theorem}{Theorem}
\newtheorem{lemma}[theorem]{Lemma}

\newtheorem{definition}[theorem]{Definition}

\newtheorem{fact}[theorem]{Fact}
\newtheorem{corollary}[theorem]{Corollary}

\newcommand{\ex}{\mathbb{E}}
\newcommand{\eat}[1]{}

\renewcommand{\norm}[1]{\left|\left|#1\right|\right|}
\newcommand{\bzero}{\textbf{0}}

\newcommand{\grad}{\nabla}

\newcommand{\psdleq}{\preccurlyeq}

\newcommand{\supp}{\text{supp}}
\newcommand{\cond}{\mathcal{E}}
\newcommand{\D}[1]{\text{d}#1}
\newcommand{\del}{\partial}
\newcommand{\Renyi}{R\'enyi\ }

\newcommand{\eps}{\ensuremath{\varepsilon}}
\newcommand{\privacyparam}{\zeta}
\usepackage[capitalize]{cleveref}

\title{Faster Differentially Private Samplers via R\'enyi Divergence Analysis of Discretized Langevin MCMC}

\author{%
  Arun Ganesh\thanks{
  Department of Electrical Engineering and Computer Sciences,
  UC Berkeley.
  \texttt{arunganesh@berkeley.edu}. 
  Supported in part by NSF Award CCF-1535989.
  Part of this work was done while the author was an intern at Google Brain.
  }
  \and
  Kunal Talwar\thanks{\texttt{kunal@kunaltalwar.org}. Part of this work was done while the author was at Google Brain.}
}

\date{}

\begin{document}
\maketitle
\begin{abstract}
Various differentially private algorithms instantiate the exponential mechanism, and require sampling from the distribution $\exp(-f)$ for a suitable function $f$. When the domain of the distribution is high-dimensional, this sampling can be computationally challenging. Using heuristic sampling schemes such as Gibbs sampling does not necessarily lead to provable privacy. When $f$ is convex, techniques from log-concave sampling lead to polynomial-time algorithms, albeit with large polynomials. Langevin dynamics-based algorithms offer much faster alternatives under some distance measures such as statistical distance. In this work, we establish rapid convergence for these algorithms under distance measures more suitable for differential privacy. For smooth, strongly-convex $f$, we give the first results proving convergence in R\'enyi divergence. This gives us fast differentially private algorithms for such $f$. Our techniques and simple and generic and apply also to underdamped Langevin dynamics.
\end{abstract}
\thispagestyle{empty}
\setcounter{page}{0}
\clearpage

\section{Introduction}

%Various differentially private algorithms instantiate the exponential mechanism, and require sampling from the distribution $\exp(-f)$ for a suitable function $f$. When the domain of the distribution is high-dimensional, this sampling can be challenging. Using heuristic sampling schemes such as Gibbs sampling does not necessarily lead to provable privacy. When $f$ is convex, techniques from log-concave sampling lead to polynomial-time algorithms, albeit with large polynomials. Langevin dynamics-based algorithms offer much faster alternatives under some distance measures such as statistical distance. In this work, we establish rapid convergence for these algorithms under distance measures more suitable for differential privacy. For smooth, strongly-convex $f$, we give the first results proving convergence in R\'enyi divergence and show stronger results under $(\eps, \delta)$-DP. This gives us fast differentially private algorithms for such $f$. Our techniques and simple and generic and apply also to underdamped Langevin dynamics.

The Exponential Mechanism~\citep{McSherryT07} is a commonly-used mechanism in differential privacy~\citep{DworkRbook}. There is a large class of mechanisms in the differential privacy literature that instantiate the Exponential Mechanism with appropriate score functions, use it as a subroutine, or sample from $\exp(-f)$ for some function $f$. This family includes differentially private mechanisms for several important problems, such as PCA~\citep{ChaudhuriSS, KapralovT}, functional PCA~\citep{AwanKRS}, answering counting queries~\citep{HardtT}, robust regression~\citep{AsiD}, some combinatorial optimization problems~\citep{GuptaLMRT}, $k$-means clustering~\citep{FeldmanFKN}, optimization of dispersed functions~\citep{BalcanDV}, convex optimization~\citep{BassilyST, MinamiASN16}, Bayesian data analysis~\citep{Mir, DimitrakakisNMR, WangFS15, WassermanZ, FouldsGWC}, linear and quantile regression~\citep{ReimherrA}, etc.

Implementing these mechanisms requires sampling from a distribution given by $\exp(-f)$ from some domain $D$, for a suitable score function $f$. When the domain $D$ is finite and small, this sampling is straightforward. Several differentially private mechanisms instantiate the exponential mechanism where $D=\mathbb{R}^d$, in which case this sampling is not straightforward.

Such sampling problems are not new and often occur in statistics and machine learning settings. The common practical approach is to use heuristic MCMC samplers such as Gibbs sampling, which often works well in problems arising in practice. However, given that convergence is not guaranteed, the resulting algorithms may not be differentially private. Indeed one can construct simple score functions on the hypercube for which the natural Metropolis chain run for any polynomial time leads to a non-private algorithm~\citep{GaneshT}. There are also well-known complexity-theoretic barriers in exactly sampling from $\exp(-f)$ if $f$ is not required to be convex. 

Several applications however involve convex functions $f$ and this is the focus of the current work. Indeed this is the problem of sampling from a log-concave distribution, which has attracted a lot of interest. Here, there are two broad lines of work. The classical results in this line of work (e.g. ~\citep{ApplegateK91, LovaszV07}) show that given an oracle for computing the function, one can sample from a distribution that is $\eps$-close\footnote{The letter $\eps$ commonly denotes the privacy parameter in DP literature, and the  distance to the target distribution in the sampling literature. Since most of the technical part of this work deals with sampling, we will reserve $\eps$ for distance, and will let $\privacyparam$ denote the privacy parameter.} 
to the target distribution in time polynomial in $d$ and $\log \frac 1 \eps$. Here the closeness is measured in statistical distance. By itself, this does not suffice to give a differentially private algorithm, as DP requires closeness in more stringent notions of distance. The fact that the time complexity is logarithmic in $\frac 1 \eps$ however allows for an exponentially small statistical distance in polynomial time. This immediately yields $(\privacyparam, \delta)$-DP algorithms, and with some additional work can also yield $\privacyparam$-DP algorithms~\citep{HardtT}. Techniques from this line of work can also sometimes apply to non-convex $f$ of interest. Indeed~\citet{KapralovT} designed a polynomial time algorithm for the case of $f$ being a Rayleigh quotient to allow for efficient private PCA.

The runtime of these log-concave sampling algorithms however involves large polynomials. A beautiful line of work has reduced the dependence (of the number of function oracle calls) on the dimension from roughly $d^{10}$ in~\citet{ApplegateK91} to $d^3$ in~\citet{LovaszV06, LovaszV07}. Nevertheless, the algorithms still fall short of being efficient enough to be implementable in practice for large $d$. A second, more recent, line of work ~\citep{Dalalyan16,DurmusM19} have shown that ``first order'' Markov Chain Monte Carlo (MCMC) algorithms such as Langevin MCMC and Hamiltonian MCMC enjoy fast convergence, and have better dependence on the dimension. These algorithms are typically simpler and more practical but have polynomial dependence on the closeness parameter $\eps$. This polynomial dependence on $\eps$ makes the choice of distance more important. Indeed these algorithms have been analyzed for various measures of distance between distributions such as statistical distance, KL-divergence and Wasserstein distance.

These notions of distance however do not lead to efficient differentially private algorithms (see~\cref{app:distance_measures}). This motivates the question of establishing rapid mixing in \Renyi divergence for these algorithms. This is the question we address in this work, and show that when $f$ is smooth and strongly convex, discretized Langevin dynamics converge in iteration complexity near-linear in the dimension. This gives more efficient differentially private algorithms for sampling for such $f$.

\citet{VempalaW19} recently studied this question, partly for similar reasons. They considered the Unadjusted (i.e., overdamped) Langevin Algorithm and showed that when the (discretized) Markov chain satisfies suitable mixing properties (e.g. Log Sobolev inequality), then the discrete process converges in \Renyi divergence to {\em a} stationary distribution. However this stationary distribution of the discretized chain is different from the target distribution. The \Renyi divergence between the stationary distribution and $\exp(-f)$ is not very well-understood~\citep{RobertsT1996, Wibisono18}, and it is conceivable that the stationary distribution of the discrete process is {\em not} close in \Renyi divergence to the target distribution and thus may not be differentially private. Thus the question of designing fast algorithms that sample from a distribution close to the distribution $\exp(-f)$ in \Renyi divergence was left open.

In this work we use a novel approach to address these questions of fast sampling from $\exp(-f)$ using the discretized Langevin Algorithm. Interestingly, we borrow tools commonly used in differential privacy, though applied in a way that is not very intuitive from a privacy point of view. We upper bound the \Renyi divergence between the output of the discrete Langevin Algorithm run for $T$ steps, and the output of the continuous process run for time $T \eta$. The continuous process is known~\citep{VempalaW19} to converge very quickly in \Renyi divergence to the target distribution. This allows us to assert closeness (in \Renyi divergence) of the output of the discrete algorithm to the target distribution.  This bypasses the question of the bias of the stationary distribution of the discrete process. Moreover, this gives us a differentially private algorithm with iteration complexity near-linear in the dimension. Our result applies to log-smooth and strongly log-concave distributions. While results of this form may also be provable using methods from optimal transport, we believe that our techniques are simpler and more approachable to the differential privacy community, and may be more easily adaptable to other functions $f$ of interest.

Our approach is general and simple. We show that it can be extended to the {\em underdamped} Langevin dynamics which have a better dependence on dimension, modulo proving fast mixing for the continuous process. As a specific application, we show how our results lead to faster algorithms for implementing the mechanisms in~\citet{MinamiASN16}.

%[Re. stationary distribution, or more generally for larger $T$: can we not argue that (A) the distribution after $\tau$ steps is close in Renyi for a suitable $\tau$. (B) if we start for this distribution, this is a good warm start so doing our analysis for the next $\tau$ steps, we maintain this property. Also, this inductively should imply that the stationary distribution has good renyi.]\arun{We use triangle inequality to show that if our starting $2\alpha$-divergence is small, the final $\alpha$-divergence of the discrete chain is small, so induction gets pretty bad}

As is common in this line of work, we ignore numerical issues and assume real arithmetic. The results can be translated to the finite-precision arithmetic case by standard techniques, as long as the precision is at least logarithmic in $d$ and $T$. The real arithmetic assumption thus simplifies the presentation without affecting the generality of the results.

\subsection{On Distance Measures between Distributions}
\label{app:distance_measures}

Existing algorithms for sampling from logconcave distributions are known to output samples from a distribution that is close to the intended distribution. The closeness is typically measured in statistical distance, Wasserstein distance, or in KL divergence. Unfortunately, none of these distances are strong enough to ensure differential privacy for the resulting algorithm. The more stringent choice of distance in differential privacy is for a good reason: it is easy to construct examples of algorithms that ensure privacy with respect to one of these weaker notions of distance but are clearly unsatisfactory from a privacy point of view~\citep{DworkRbook}. This motivates the question of efficient sampling in terms of a stronger measure of distance such as $\infty$-divergence, or \Renyi divergence (both of which upper bound KL divergence and thus upper bound statistical distance and Wasserstein distances). Different distance notions can be related to each other and~\citet{HardtT} showed that an exponentially small statistical distance guarantee suffices to derive a differentially private algorithm. This allows for polynomial time algorithms using the classical logconcave samplers.

The faster sampling algorithms based on Langevin dynamics and relatives however have a polynomial dependence on the distance. In this case, convergence under the various notions of distance is not equivalent. None of the commonly used measures (Statistical distance, KL-divergence or Wasserstein distance) can be polynomially related to common distances of interest from a privacy point-of-view ($\infty$-divergence, \Renyi divergence). While $(\privacyparam, \delta)$-DP can be related via a polynomial in $\delta^{-1}$, this would lead to algorithms that have runtime polynomial in $\delta^{-1}$, which is undesirable as we often want $\delta^{-1}$ to be super-polynomial.

\subsection{Other Related Work} %% Not sure if we want to put too much else here so making it a sentence.
\citet{WangFS15} discuss the issue of privacy when using approximate samplers at length and consider two algorithms. The first one (OPS) that samples approximately from $\exp(-f)$ considers closeness in statistical distance and thus can only be efficient when coupled with the first kind of samplers above, i.e. those that have a logarithmic dependence on the closeness parameter. The second algorithm they analyze is a variant of Stochastic Gradient Langevin Dynamics (SGLD). The algorithm adds additional noise for privacy, and while it is shown to be private for suitable parameters, it does not ensure convergence to the target distribution. Differentially private approximations to SGLD have also been studied in~\cite{LiCLC}. Note that in contrast, we do not need to modify the Langevin dynamics which ensures convergence as well as privacy.

There is a large body of work on Langevin algorithms and their variants. We refer the reader to the surveys by~\citet{roberts2004} and~\citet{Vempala05geometricrandom}. There has been a recent spate of activity on analyzing these algorithms and their stochastic variants, under different kinds of assumptions on $f$ and we do not attempt to summarize it here.

\subsection{Results and Techniques}
\begin{figure}
\begin{center}
\begin{tabular}{|l|l|c|c|}
\hline
$f$ is $L$-smooth and & Process & $\eta$ & Iterations\\
\hline
1-strongly convex & Overdamped  & $\tilde{O}\left(\frac{1}{\tau L^4 \ln^2 \alpha} \cdot \frac{\eps^2}{d} \right)$ (Thm~\ref{lemma:conditionaldivergence}) & $\tilde{O}\left(\frac{d \tau^2 L^4 \ln^2 \alpha}{\eps^2} \right)$ \\
\hline
$B$-Lipschitz     & Overdamped  & $\tilde{O}\left(\frac{1}{\tau L^4 \ln^2 \alpha} \cdot \frac{\eps^2}{B^2 + d} \right)$ (Thm~\ref{lemma:conditionaldivergence-lipschitz})& $\tilde{O}\left(\frac{(B^2 + d) \tau^2 L^4 \ln^2 \alpha}{\eps^2} \right)$ \\
\hline
1-strongly convex & Underdamped &  $\tilde{O}\left(\frac{1}{\tau L \ln \alpha} \cdot \frac{\eps}{\sqrt{d}} \right)$ (Thm~\ref{lemma:conditionaldivergence-ud}) & $\tilde{O}\left(\frac{\sqrt{d} \tau^2 L \ln \alpha}{\eps} \right)$ \\
\hline
\end{tabular}
\end{center}
\caption{Summary of results. For each family of functions and process (either overdamped or underdamped Langevin dynamics), an upper bound is listed on the step size $\eta$ (and thus a bound on the iteration complexity) needed to ensure the $\alpha$-R\'enyi divergence between the discrete and continuous processes is at most $\eps$ after time $\tau$. Setting $\alpha = O(\ln(1/\delta)/\privacyparam), \eps = \privacyparam/2$ gives that the $\delta$-approximate max divergence is at most $\privacyparam$, i.e. $(\privacyparam, \delta)$-differential privacy.}
\label{fig:results}
\end{figure}

Our results are summarized in Figure~\ref{fig:results}. Combined with results from \cite{VempalaW19} on the convergence of the continuous process, the first result gives the following algorithmic guarantee, our main result:

\begin{theorem}\label{thm:maindpthm}
Fix any $\alpha \geq 1$. Let $R$ be a distribution satisfying $R(x) \propto e^{-f(x)}$ for 1-strongly convex and $L$-smooth $f$ with global minimum at 0. Let $P$ be the distribution arrived at by running discretized overdamped Langevin dynamics using $f$ with step size $\eta = \tilde{O}(\frac{1}{\tau L^4 \ln^2 \alpha} \cdot \frac{\eps^2}{d})$ for continuous time $\tau = O(\alpha \ln \frac{d \ln L}{\eps})$ (i.e. for $\tilde{O}(\frac{\alpha^2 L^4 d}{\eps^2})$ steps) from initial distribution $N(0, I_d)$. Then we have $D_\alpha(P || R), D_\alpha(R || P) \leq \eps$.
\end{theorem}

This is the first algorithmic result for sampling from log-smooth and strongly log-concave distributions with low error in R\'enyi divergence without additional assumptions. In particular, if for $\alpha = 1 + 2\log(1/\delta)/\privacyparam$ we have $D_\alpha(P||R), D_\alpha(R||P) \leq \privacyparam/2$, then by Fact~\ref{fact:renyitoapx} we have that $P, R$ satisfy the divergence bounds of $(\privacyparam, \delta)$-differential privacy. In turn, given any mechanism that outputs $R, R'$ on adjacent databases satisfying $(\privacyparam, \delta)$-differential privacy and the strong convexity and smoothness conditions, Theorem~\ref{thm:maindpthm} and standard composition theorems gives a mechanism that outputs $P, P'$ for these databases such that the mechanism satisfies $(3\privacyparam, 3\delta)$-differential privacy, $P, P'$ are efficiently sampleable, and $P, P'$ obtain utility guarantees comparable to those of $R, R'$.

All results in Figure~\ref{fig:results} are achieved using a similar analysis, which we describe here.
Instead of directly bounding the divergence between the discrete and continuous processes, we instead bound the divergence between the discrete processes using step sizes $\eta, \eta/k$. Our resulting bound does not depend on $k$, so we can take the limit as $k$ goes to infinity and the latter approaches the continuous process. Suppose within each step of size $\eta$, neither process moves more than $r$ away from the position at the start of this step. Then by smoothness, in each interval of length $\eta/k$ the distance between the gradient steps between the two processes is upper bounded by $Lr \frac{\eta}{k}$. Our divergence bound thus worsens by at most $D_\alpha(N(0, \frac{2\eta}{k}) ||  N(x, \frac{2\eta}{k}))$ where $x$ is a vector with $\norm{x}_2 \leq Lr \frac{\eta}{k}$. The divergence between shifted Gaussians is well-known, giving us a divergence bound. 

Of course, since the movement due to Brownian motion can be arbitrarily large, there is no unconditional bound on $r$. Instead, we derive tail bounds for $r$, giving a divergence bound (depending on $\delta$) between the two processes conditioned on a probability $1 - \delta$ event for every $\delta$. We then show a simple lemma which says that conditional upper bounds on the larger moments of a random variable give an unconditional upper bound on the expectation of that random variable. By the definition of R\'enyi divergence, $\exp((\alpha' - 1) D_{\alpha'}(P || Q))$ is a moment of $\exp((\alpha - 1) D_{\alpha}(P || Q))$ for $\alpha' > \alpha$, so we can apply this lemma to our conditional bound on $\alpha'$-R\'enyi divergence to get an unconditional bound on $\alpha$-R\'enyi divergence via Jensen's inequality. 

Finally, since our analysis only needs smoothness, the radius tail bound, and the fact that the process is a composition of gradient steps with Gaussian noise, our analysis easily extends to sampling from Lipschitz rather than strongly convex functions and analyzing the underdamped Langevin dynamics.

As an immediate application, we recall the work of~\citet{MinamiASN16}, who give a $(\privacyparam, \delta)$-differentially private mechanism that (approximately) samples from a Gibbs posterior with a strongly log-concave prior, for applications such as mean estimation and logistic regression. Their iteration complexity of $\tilde{O}(d^3/\delta^2)$ proved in~\citet[Prop. 13]{MinamiASN16} gets improved to $\tilde{O}(d/\privacyparam^4)$ using our main result. We note that the privacy parameters in $(\privacyparam, \delta)$-DP that one typically aims for are $\privacyparam$ being constant, and $\delta$ being negligible.  However, it is still an interesting open problem to improve the iteration complexity's dependence on $\privacyparam$.

We start with some preliminaries in~\cref{section:prelim}. We prove the main result in~\cref{section:bounded} ,~\cref{section:odconvergence}, and~\cref{section:bidirectional}, and prove the result for the underdamped case in~\cref{sec:underdamped}. We defer the proofs of some tail bounds to~\cref{section:tailbounds}. We discuss future research directions in~\cref{section:discussion}. 
\section{Preliminaries}
\label{section:prelim}
\subsection{Langevin Dynamics and Basic Assumptions}

For the majority of the paper we focus on the overdamped Langevin dynamics in $\mathbb{R}^d$, given by the following stochastic differential equation (SDE):
$$ \D{x_t} = -\grad f(x_t) \D{t} + \sqrt{2} \D{B_t},$$
Where $B_t$ is a standard $d$-dimensional Brownian motion. Under mild assumptions (such as strong convexity of $f$), it is known that the stationary distribution of the SDE is the distribution $p$ satisfying $p(x) \propto e^{-f(x)}$. Algorithmically, it is easier to use the following discretization with \textit{steps} of size $\eta$:
$$ \D{x_t} = -\grad f(x_{\lfloor \frac{t}{\eta}\rfloor \eta}) \D{t} + \sqrt{2} \D{B_t},$$
i.e., we only update the gradient used in the SDE at the beginning of each step. Restricted to the position at times that are multiples of $\eta$, equivalently:
$$x_{(i+1)\eta} = x_{i \eta} - \eta \grad f(x_{i \eta}) + \xi_i.$$
Where $\xi_i \sim N(0, 2\eta I_d)$ are independent samples. Throughout the paper, when we refer to the result of running a Langevin dynamics for \textit{continuous time $t$}, we mean the distribution $x_t$, \textit{not} the distribution $x_{t \eta}$. When the iteration complexity (i.e. number of steps) is of interest, we may refer to running a Langevin dynamics for continuous time $T\eta$ equivalently as the result of running it for $T$ steps (of size $\eta$).

A similarly defined second order process is the underdamped Langevin dynamics, given by the following SDE (parameterized by $\gamma, \mu > 0$):
$$\D v_t = -\gamma v_t \D t - \mu \grad f(x_t) \D t + \sqrt{2 \gamma \mu} \D B_t, \qquad \D x_t = v_t \D t.$$
Again, under mild assumptions it is known that the stationary distribution of this SDE is the distribution $p$ satisfying $p(x) \propto e^{-(f(x) + \norm{v}_2^2/2\mu)}$, so that the marginal on $x$ is as desired. Algorithmically, it is easier to use the following discretization:
\begin{equation}\label{eq:discreteud}
    \D v_t = -\gamma v_t \D_t - \mu \grad f(x_{\lfloor \frac{t}{\eta} \rfloor \eta}) \D t + \sqrt{2 \gamma \mu} \D B_t, \qquad \D x_t = v_t \D t.
\end{equation}
In the majority of the paper we consider sampling from distributions given by $m$-strongly convex, $L$-smooth functions $f$. To simplify the presentation, we also assume $f$ is twice-differentiable, so these conditions on $f$ can be expressed as: for all $x$, $mI \psdleq \grad^2 f(x) \psdleq LI$. We make two additional simplifying assumptions: The first is that the minimum point of $f$ is at $0$, as if $f$'s true minimum is $x^* \neq 0$, we can sample from $g(x) := f(x - x^*)$ and then shift our sample by $x^*$ to get a sample from $f$ instead ($x^*$ can be found using e.g. gradient descent). The second is that $m = 1$, as if $m \neq 1$, we can sample from $g(x) = f(\frac{1}{\sqrt{m}}x)$ and rescale our sample by $\sqrt{m}$ instead.

\subsection{R\'enyi Divergence}

We recall the definition of R\'enyi divergence:

\begin{definition}[R\'enyi Divergence]
For $0 < \alpha < \infty$, $\alpha \neq 1$ and distributions $\mu, \nu$, such that $\supp(\mu) = \supp(\nu)$ the $\alpha$-R\'enyi divergence between $\mu$ and $\nu$ is
$$ D_\alpha(\mu || \nu) = \frac{1}{\alpha - 1} \ln \int_{\supp(\nu)} \frac{\mu(x)^\alpha}{\nu(x)^{\alpha - 1}} \D{x} = \frac{1}{\alpha - 1} \ln \ex_{x \sim \mu}\left[ \frac{\mu(x)^{\alpha-1}}{\nu(x)^{\alpha-1}}\right] = \frac{1}{\alpha - 1} \ln \ex_{x \sim \nu}\left[ \frac{\mu(x)^\alpha}{\nu(x)^\alpha}\right].$$
The $\alpha$-R\'enyi divergence for $\alpha = 1$ (resp. $\infty$) is defined by taking the limit of $D_\alpha(\mu||\nu)$ as $\alpha$ approaches $1$ (resp. $\infty$) and equals the KL divergence (resp. max divergence). 

The definition of $\alpha$-R\'enyi divergence can be extended to negative $\alpha$ using the identity $D_{1-\alpha}(\mu||\nu) = \frac{1-\alpha}{\alpha} D_{\alpha}(\nu||\mu)$.
\end{definition}

Throughout the paper, we are often concerned with pairs of distributions whose supports are both $\mathbb{R}^d$, and so we will use the above definition without always stating this fact explicitly. R\'enyi divergence is a standard notion of divergence in information theory. The following properties of R\'enyi divergences are useful in our proofs:

\begin{fact}[Monotonicity {\cite[Theorem 3]{vanErvenH14}}]
For any distributions $P, Q$ and $\alpha_1 \leq \alpha_2$ we have $D_{\alpha_1}(P || Q)$ $\leq D_{\alpha_2}(P||Q)$. 
\end{fact}

\begin{fact}[Post-Processing {\cite[Theorem 9]{vanErvenH14}}]\label{fact:postprocessing}
For any sample spaces $\mathcal{X}, \mathcal{Y}$, distributions $X_1, X_2$ over $\mathcal{X}$, and any function $f:\mathcal{X} \rightarrow \mathcal{Y}$ we have $D_\alpha(f(X_1) ||f(X_2)) \leq D_\alpha(X_1 || X_2)$.
\end{fact}

The above is also known as the data processing inequality in information theory.

\begin{fact}[Gaussian Divergence {\cite[Example 3]{vanErvenH14}}]\label{fact:gaussiandivergence}
$$D_\alpha(N(0, \sigma^2 I_d) || N(x, \sigma^2 I_d)) \leq \frac{\alpha \norm{x}_2^2}{2\sigma^2}.$$
\end{fact}

\begin{fact}[Adaptive Composition Theorem {\cite[Proposition 1]{Mironov2017}}]\label{fact:composition}
Let $\mathcal{X}_0,$ $\mathcal{X}_1, \ldots, \mathcal{X}_k$ be arbitrary sample spaces. For each $i \in [k]$, let $\psi_i, \psi_i':\Delta(\mathcal{X}_{i-1}) \rightarrow \Delta(\mathcal{X}_i)$ be maps from distributions over $\mathcal{X}_{i-1}$ to distributions over $\mathcal{X}_i$ such that for any point mass distribution (a distribution whose support contains a single value) $X_{i-1}$ over $\mathcal{X}_{i-1}$, $D_\alpha(\psi_i(X_{i-1}) || \psi_i'(X_{i-1})) \leq \eps_i$. Then, for $\Psi, \Psi':\Delta(\mathcal{X}_0) \rightarrow \Delta(\mathcal{X}_k)$ defined as $\Psi(\cdot) = \psi_k(\psi_{k-1}( \ldots \psi_1(\cdot) \ldots )$ and $\Psi'(\cdot) = \psi'_k(\psi'_{k-1}( \ldots \psi'_1(\cdot) \ldots )$ we have $D_\alpha(\Psi(X_0) || \Psi'(X_0)) \leq \sum_{i=1}^k \eps_i$ for any $X_0 \in \Delta(\mathcal{X}_0)$.
\end{fact}

\begin{fact}[Weak Triangle Inequality {\cite[Proposition 11]{Mironov2017}}]\label{fact:triangleineq}
For any $\alpha > 1$, $p, q > 1$ satisfying $1/p + 1/q = 1$ and distributions $P, Q, R$ with the same support:

$$D_\alpha(P||R) \leq \frac{\alpha - 1/p}{\alpha - 1}D_{p\alpha}(P||Q) + D_{q(\alpha - 1/p)}(Q||R).$$
\end{fact}

Unlike other notions of distance between distributions, R\'enyi divergence bounds translate to differential privacy guarantees:
\begin{definition}[Approximate Differential Privacy]\label{definition:apxdp}
The $\delta$-approximate max divergence between distributions $\mu, \nu$ is defined as:
$$D_\infty^\delta(\mu || \nu) = \max_{S \subseteq \supp(\mu): \Pr_{x \sim \mu}[x \in S] \geq \delta} \left[\ln \frac{\Pr_{x \sim \mu}[x \in S] - \delta}{\Pr_{x \sim \nu}[x \in S]}\right].$$
\end{definition}
\begin{fact}[{\cite[Proposition 3]{Mironov2017}}]\label{fact:renyitoapx}
For $\alpha > 1$ if $\mu, \nu$ satisfy $D_\alpha(\mu || \nu) \leq \privacyparam,$ then for $0 < \delta < 1$:
$$D_\infty^\delta (\mu || \nu) \leq \privacyparam + \frac{\ln(1/\delta)}{\alpha - 1}.$$
\end{fact}

\section{Langevin Dynamics with Bounded Movements}~\label{section:bounded}
As a first step, we analyze the divergence between the discrete and continuous processes conditioned on the event $\cond_r$ that throughout each step of size $\eta$ they stay within a ball of radius $r$ around their location at the start of the step. We will actually analyze the divergence between two discrete processes with steps of size $\eta$ and $\eta/k$ respectively, and obtain a bound on their divergence independent of $k$. The former is exactly the discrete Langevin dynamics with step size $\eta$. The Taking the limit of the latter, as $k$ goes to infinity, the former is exactly the discrete Langevin dynamics with step size $\eta$ and the latter is the continuous Langevin dynamics. Thus, and so the same bound applies to the divergence between the discrete and continuous processes. We set up discretized overdamped Langevin dynamics with step sizes $\eta, \eta/k$ as random processes which record the position at each time that is a multiple of $\eta/k$. 

Let $x_t$ denote the position of the chain using step size $\eta$ at continuous time $t$, and $x_t'$ denote the position of the chain using step size $\eta/k$ at time $t$. If $\cond_r$ does not hold at time $t^*$ (more formally, if $\max_{t \in [0, t^*]} \norm{x_{t} - x_{\lfloor {t} / \eta \rfloor \eta}}_2 > r$), we will instead let $x_t = \bot$ for all $t \geq t^*$. We want to bound the divergence after $T$ steps of size $\eta$, i.e. the divergence between the distributions of $x_{T \eta}$ and $x_{T \eta}'$. Let $X_{0 : j}$ denote the distribution of $\{x_{i \eta /k}\}_{0 \leq i \leq j}$, and define $X_{0:j}'$ analogously. By post-processing  (Fact~\ref{fact:postprocessing}), it suffices to bound the divergence between $X_{0:Tk}$ and $X_{0:Tk}'$. Note that we can sample from $X_{0:Tk}$ (resp $X_{0:Tk}'$) by starting with a sample $\{x_0\}$ (resp $\{x_0'\}$) from the distribution $X_0$ from which we start the Langevin dynamics, and applying the following randomized update $Tk$ times: 

\begin{itemize}
    \item To draw a sample from $X_{0:Tk}$, given a sample $\{x_{i\eta/k}\}_{0 \leq i \leq j}$ from $X_{0:j}$:
    \begin{itemize}
        \item If $x_{j\eta/k} = \bot$ append $x_{(j+1)\eta/k} = \bot$ to $\{x_{i\eta/k}\}_{0 \leq i \leq j}$ to get a sample from $X_{0:j+1}$.
        \item Otherwise, append $x_{(j+1)\eta/k} = x_{j\eta/k} - \frac{\eta}{k} \grad f(x_{\lfloor j / k \rfloor \eta}) + \xi_j$, where $\xi_j \sim N(0,\frac{2\eta}{k}I_d)$ to get a sample from $X_{0:j+1}$. Then if $\norm{x_{(j+1)\eta/k} - x_{\lfloor (j+1)/k\rfloor \eta}}_2 > r$ (i.e. $\cond_r$ no longer holds) replace $x_{(j+1)\eta/k}$ with $\bot$.
    \end{itemize}  
    We will denote this update by $\psi$. More formally, $\psi$ is the map from distributions over to distributions such that $X_{0:j+1} = \psi(X_{0:j})$.
    \item To draw a sample from $X_{0:Tk}'$, we instead use the update $\psi'$ that is identical to $\psi$ except $\psi'$ uses the gradient at $x'_{j \eta / k}$ instead of $x'_{\lfloor j / k \rfloor \eta}$.
\end{itemize}

We now have $X_{0:Tk} = \psi^{\circ Tk}(X_0)$ and $X_{0:Tk}' = (\psi')^{\circ Tk}(X_0)$, allowing us to use Fact~\ref{fact:composition} to bound the divergence between the two distributions:

\begin{lemma}\label{lemma:smallstepconvergence}
For any $L$-smooth $f$, any initial distribution $X_0$ over $x_0, x_0'$, and the distributions over tuples $X_{0:Tk}, X'_{0:Tk}$ as defined above, we have:
$$D_\alpha(X_{0:Tk} || X_{0:Tk}'), D_\alpha(X_{0:Tk}' || X_{0:Tk}) \leq \frac{T\alpha L^2 r^2 \eta}{4}.$$
\end{lemma}
\begin{proof}
We prove the bound for $D_\alpha(X_{0:Tk} || X_{0:Tk}')$, the bound for $D_\alpha(X_{0:Tk}' || X_{0:Tk})$ follows similarly. Let a tuple $\{x_{i\eta/k}\}_{0 \leq i \leq j}$ be \textit{good} if for $0 \leq i \leq j$ either (i) $\norm{x_{i\eta/k} - x_{\lfloor i/k\rfloor \eta}}_2 \leq r$ (i.e., $\cond_r$) or (ii) $\{x_{\ell\eta/k}\}_{i \leq \ell \leq j}$ are all $\bot$. We claim that for each $j$, for any point mass distribution $X_{0:j}$ over good $(j+1)$-tuples:

\begin{equation}\label{eq:singlemapbound}
D_\alpha(\psi(X_{0:j}), \psi'(X_{0:j})) \leq \frac{\alpha (\frac{Lr\eta}{k})^2}{2 \cdot \frac{2\eta}{k}}.
\end{equation}

By Fact~\ref{fact:postprocessing}, we can instead bound the divergence between $\tilde{\psi}(X_{0:j}), \tilde{\psi}'(X_{0:j})$ which are defined equivalently to $\psi, \psi'$ except without the deterministic step of replacing the last entry with $\bot$ if $\cond_r$ is violated. If $X_{0:j}$ is a point mass on a good tuple containing $\bot$, then $D_\alpha(\tilde{\psi}(X_{0:j})|| \tilde{\psi}'(X_{0:j})) = 0$. For $X_{0:j}$ that is a point mass on a good tuple not containing $\bot$, $D_\alpha(\tilde{\psi}(X_{0:j})|| \tilde{\psi}'(X_{0:j}))$ is just the divergence between the final values of $\tilde{\psi}(X_{0:j}), \tilde{\psi}'(X_{0:j})$. The distance between the final values in $\tilde{\psi}(X_{0:j}), \tilde{\psi}'(X_{0:j})$ prior to the addition of Gaussian noise in $\tilde{\psi}, \tilde{\psi}'$ is the value of $\frac{\eta}{k}\norm{\grad f(x_{j \eta / k}) - \grad f(x_{\lfloor j / k \rfloor \eta})}_2$ for the single tuple in the support of $X_{0:j}$, which is at most $\frac{Lr\eta}{k}$ by smoothness and because $\cond_r$ holds for all good tuples not containing $\bot$. \eqref{eq:singlemapbound} now follows by Fact~\ref{fact:gaussiandivergence}.

Then, $X_{0:Tk}, X'_{0:Tk}$ are arrived at by a composition of $Tk$ applications of $\psi, \psi'$ to the same initial distribution $X_0$. Note that $X_0$ and the distributions arrived at by applying $\psi$ or $\psi'$ any number of times to $X_0$ have support only including good tuples. Then combining Fact~\ref{fact:composition} (with the sample spaces being good tuples) and \eqref{eq:singlemapbound} we have:

$$D_\alpha(X_{0:Tk} || X'_{0:Tk}) \leq Tk \cdot \frac{\alpha \left(\frac{Lr\eta}{k}\right)^2}{2 \cdot \frac{2\eta}{k}} = \frac{T\alpha L^2 r^2 \eta}{4}.$$
\end{proof}
By taking the limit as $k$ goes to infinity and applying Fact~\ref{fact:postprocessing} we get:
\begin{corollary}\label{cor:smallstepconvergence}
For any $L$-smooth $f$ and $\eta > 0$, and any initial distribution $X_0$ let $X_{t}$ be the distribution over positions $x_t$ arrived at by running the discretized overdamped Langevin dynamics with step size $\eta$ on $f$ from $X_0$ for continuous time $t$, except that $X_{t} = \bot$ if $\cond_r$ does not hold at time $t$ for this chain. Let $X_{t}'$ be the same but for the continuous overdamped Langevin dynamics. Then for any integer $T \geq 0$:

$$D_\alpha(X_{T\eta} || X_{T\eta}'), D_\alpha(X_{T\eta}' || X_{T\eta}) \leq \frac{T\alpha L^2 r^2 \eta}{4}.$$
\end{corollary}

Note that if we are running the process for continuous time $\tau$, then $T = \tau/\eta$. $r$ will end up being roughly proportional to $\sqrt{\eta}$, so the above bound is then roughly proportional to $\eta$.
\section{Removing the Bounded Movement Restriction}\label{section:odconvergence}

In this section, we will prove the following ``one-sided'' version of Theorem~\ref{thm:mainthm}:

\begin{theorem}\label{thm:mainthm}
Fix any $\alpha \geq 1$. Let $R$ be a distribution satisfying $R(x) \propto e^{-f(x)}$ for 1-strongly convex and $L$-smooth $f$ with global minimum at 0. Let $P$ be the distribution arrived at by running discretized overdamped Langevin dynamics using $f$ with step size $\eta = \tilde{O}(\frac{1}{\tau L^4 \ln^2 \alpha} \cdot \frac{\eps^2}{d})$ for continuous time $\tau = \alpha \ln \frac{d \ln L}{\eps}$ (i.e. for $\tilde{O}(\frac{\alpha^2 L^4 d}{\eps^2})$ steps) from initial distribution $N(0, \frac{1}{L}I_d)$. Then we have $D_\alpha(P || R) \leq \eps$.
\end{theorem}

To remove the assumption that the process never moves more than $r$ away from its original position within each step of size $\eta$, we give a tail bound on the maximum value $r$ that the process moves within one of these steps. 
% that holds with probability $1 - \delta$. 

\begin{lemma}\label{lemma:radiustailbound}
Let $c$ be a sufficiently large constant. Let $\eta \leq \frac{2}{L+1}$ and let $X_0$ be an initial distribution over $\mathbb{R}^d$ satisfying that for all $\delta > 0$,
\begin{equation}\label{eq:startingbound}
    \Pr_{x \sim X_0}\left[\norm{x}_2 \leq \frac{c}{2 \sqrt{\eta}} \left(\sqrt{d} + \sqrt{\ln(T/\delta)}\right)\right] \geq 1 - \frac{\delta}{4(T+1)}.
\end{equation}
%initial distribution $N(0, \frac{1}{L}I_d)$,
Let $x_t$ be the random variable given by running the discretized overdamped Langevin dynamics starting from $X_0$ for continuous time $t$. Then with probability at least $1 - \delta$ over the path $\{x_t : t \in [0, T\eta]\}$:
$$\forall t \leq T\eta: \norm{x_t - x_{\lfloor t / \eta \rfloor \eta}}_2\leq cL\left(\sqrt{d} + \sqrt{\ln(T/\delta)}\right)\sqrt{\eta}.$$
Similarly, let $x'_t$ be the random variable given by running the continuous overdamped Langevin dynamics starting from $X_0$ for continuous time $t$. Then with probability at least $1-\delta$ over the path $\{x'_t : t \in [0, T\eta]\}$:
$$\forall t \leq T\eta: \norm{x'_t - x'_{\lfloor t / \eta \rfloor \eta}}_2\leq cL\left(\sqrt{d} + \sqrt{\ln(T/\delta)}\right)\sqrt{\eta}.$$
\end{lemma}

The proof is deferred to Section~\ref{section:tailbounds}. Intuitively, the $\sqrt{\eta}$ accounts for movement due to Brownian motion, which dominates the movement due to the gradient, and $cL(\sqrt{d} + \sqrt{\ln(T/\delta)})$ is a tail bound on norm of the gradient by smoothness. 
This gives us a bound on the R\'enyi divergence between the continuous and discrete processes conditioned on a probability $1-\delta$ event for all $0 < \delta < 1$.
By absorbing the failure probability of this event into the probability of large privacy loss in the definition of $(\privacyparam, \delta)$-differential privacy we can prove iteration complexity bounds matching those in Figure~\ref{fig:results} for running discretized overdamped Langevin dynamics with $(\privacyparam, \delta)$-differential privacy  without using the tools we develop in the rest of this section. Since these bounds do not improve on those in the ones derived from our final (unconditional) divergence bounds, we omit the proof here.
%
% \begin{lemma}\label{lemma:apxdp}
% Let $P$ be the distribution over $\mathbb{R}^d$ arrived at by running the discretized overdamped Langevin dynamics for $T$ timesteps of length $\eta$ (i.e. continuous time $\tau := T\eta$) for a $1$-strongly convex, $L$-smooth function $f$, and $Q$ be the same as $P$ but for the continuous overdamped Langevin dynamics. Then for any $0 < \delta < 1$ and $\privacyparam > 0$, if $\eta = \tilde{O}(\frac{\privacyparam^2}{\tau L^4 \ln(1/\delta) \max\{d, \privacyparam \ln(1/\delta)\}})$, (i.e. the discrete chain runs for $\tilde{O}(\frac{\tau^2 L^4 \ln(1/\delta) \max\{d, \privacyparam \ln(1/\delta)\}}{\privacyparam^2})$ iterations):
%
% $$D_\infty^\delta (P || Q), D_\infty^\delta (Q || P) \leq \privacyparam.$$
% \end{lemma}
%
% In the appendix we also give bounds analogous to Lemma~\ref{lemma:apxdp} for the settings we consider in the rest of the paper. 

To prove a R\'enyi divergence bound, we need to remove the conditioning. We start with the following lemma, which takes bounds on conditional moments and gives an unconditional bound on expectation:

\begin{lemma}\label{lemma:expectationfromconditional}
Let $Y$ be a random variable distributed over $\mathbb{R}_{\geq 0}$ that has the following property (parameterized by positive parameters $\beta, \gamma < 1, \theta > 1 + \gamma$): For every $0 < \delta < 1/2$, there is a probability at least $1 - \delta$ event $\mathcal{E}_\delta$ such that $\ex\left[Y^\theta | \mathcal{E}_\delta \right] \leq \frac{\beta}{\delta^\gamma}$. Then we have:
$$\ex[Y] \leq \beta^{\frac{1}{\theta}}\left(\gamma^{\frac{1}{1+\gamma}} + \gamma^{-\frac{\gamma}{1+\gamma}}\right)^{\frac{1+\gamma}{\theta}}\left(\frac{\theta(1+\gamma)}{\theta(1+\gamma)-1}\right) \leq \beta^{1/\theta} 2^{2/\theta} \frac{\theta}{\theta - 1}.$$
\end{lemma}

\begin{proof}
Let $z$ be an arbitrary parameter, $\eta:[z, \infty) \rightarrow (0, 1/2)$ be an arbitrary map, and $\mathcal{E}_\delta$ be the event specified in the lemma statement for $\delta \in (0, 1)$. Using the definition of expectation and the property of $Y$ in the lemma statement, we have:

\begin{align*}
\ex[Y] &= \int_0^\infty \Pr[Y \geq y] \D{y}  \\
&\leq \int_0^z 1\ \D{y} + \int_z^\infty \Pr[Y \geq y] \D{y} \\
&\leq z+ \int_z^\infty \eta(y) + (1-\eta(y))\Pr[Y \geq y | \cond_{\eta(y)}] \D{y}\\
&\leq z + \int_z^\infty \eta(y) + \Pr[Y \geq y | \cond_{\eta(y)}] \D{y}\\
&=z + \int_z^\infty \eta(y) + \Pr[Y^\theta \geq y^\theta | \cond_{\eta(y)}] \D{y} \\
&\leq z + \int_z^\infty \eta(y) + \frac{\ex[Y^\theta | \cond_{\eta(y)}]}{y^\theta} \D{y}\\
&\leq 
z + \int_z^\infty \eta(y) + \frac{\beta}{ \eta(y)^\gamma y^\theta} \D{y}.
\end{align*}

We now choose $\eta(y) = \left(\frac{\gamma \beta}{y^\theta}\right)^{\frac{1}{1+\gamma}}$ to minimize the value of the expression in the integral. We will eventually choose $z$ such that $0 < \eta(y) < 1/2$ for all $y \geq z$ as promised. Plugging in this choice of $\eta$ gives the upper bound:

\begin{align*}
\ex[Y] &\leq z +  \beta^{\frac{1}{1+\gamma}}(\gamma^{\frac{1}{1+\gamma}} + \gamma^{-\frac{\gamma}{1+\gamma}})\int_z^\infty y^{-\frac{\theta}{1+\gamma}} \D{y}\\
&= z + \beta^{\frac{1}{1+\gamma}}(\gamma^{\frac{1}{1+\gamma}} + \gamma^{-\frac{\gamma}{1+\gamma}})\left(\frac{1}{\frac{\theta}{1+\gamma}-1}\right)\left[y^{1-\frac{\theta}{1+\gamma}}\right]_\infty^z\\
&=
z + \beta^{\frac{1}{1+\gamma}}(\gamma^{\frac{1}{1+\gamma}} + \gamma^{-\frac{\gamma}{1+\gamma}})\left(\frac{1}{\frac{\theta}{1+\gamma}-1}\right)z^{1-\frac{\theta}{1+\gamma}}.
\end{align*}

We finish by choosing $z = \beta^{\frac{1}{\theta}}\left(\gamma^{\frac{1}{1+\gamma}} + \gamma^{-\frac{\gamma}{1+\gamma}}\right)^{\frac{1+\gamma}{\theta}}$. This gives the upper bound on $\ex[Y]$ in the lemma statement. We also verify that $\eta(y)$ is a map to $(0, 1/2)$: $\eta(y) \propto y^{-\frac{\theta}{1+\gamma}}$, giving that $\eta(y) > 0$. For all $y \geq z$, since $\gamma < 1$ we have $\eta(y) \leq \eta(z) =\frac{\gamma}{\gamma+1} < 1/2$.
\end{proof}

Putting it all together, we get the following lemma:

\begin{theorem}\label{lemma:conditionaldivergence}
For any $1$-strongly convex, $L$-smooth $f$, let $P$ be the distribution of states for discretized overdamped Langevin dynamics with step size $\eta$ and $Q$ be the distribution of states for continuous overdamped Langevin dynamics, both run from any initial distribution $X_0$ satisfying \eqref{eq:startingbound} for continuous time $\tau$ that is a multiple of $\eta$ (i.e. for $\tau / \eta$ steps). Then for $\alpha > 1$, $\eps > 0$, if $\eta = \tilde{O}(\frac{1}{\tau L^4 \ln^2 \alpha} \cdot \frac{\eps^2}{d})$ we have $D_\alpha(P || Q), D_\alpha(Q || P)\leq \eps$.
\end{theorem}
% \begin{theorem}\label{lemma:conditionaldivergence}
% For any $1$-strongly convex, $L$-smooth stationary distribution proportional to $e^{-f}$, let $P$ be the distribution of states for discretized overdamped Langevin dynamics with timestep $\eta$ and $Q$ be the distribution of states for continuous overdamped Langevin dynamics, both run from any initial distribution $p$ satisfying \eqref{eq:startingbound} for continuous time $\tau$ that is a multiple of $\eta$ (i.e. for $\tau / \eta$ timesteps). Then for $\alpha > 1$, $\eps > 0$, if $\eta = \tilde{O}(\frac{1}{\tau L^4 \ln^2 \alpha} \cdot \frac{\eps^2}{d})$ we have $D_\alpha(P || Q) \leq \eps$.
% \end{theorem}

We provide some high level intuition for the proof here. Plugging Lemma~\ref{lemma:radiustailbound} into Lemma~\ref{lemma:smallstepconvergence} gives a bound on roughly the $\alpha'$-R\'enyi divergence between $P$ conditioned on some probability $1-\delta_1$ event and $Q$ conditioned on some probability $1-\delta_2$ event for every $\delta_1, \delta_2$. We apply Lemma~\ref{lemma:expectationfromconditional} once for $P$ and once for $Q$ to remove the conditioning, giving a bound of $\approx \frac{\ln \alpha'}{\alpha' - 1}$ on the actual $\alpha'$-R\'enyi divergence between $P, Q$ if  $\eta$ is sufficiently small (as a function of $\alpha'$). Using Jensen's inequality, we can turn this into a bound of $\eps$ on the $\alpha$-R\'enyi divergence between $P, Q$ for any $\alpha$ if $\alpha'$ is large enough (which in turn requires $\eta$ to be small enough).

\begin{proof}[Proof of Theorem~\ref{lemma:conditionaldivergence}]
We prove the bound on $D_\alpha(P||Q)$. Since Corollary~\ref{cor:smallstepconvergence} provides a ``bi-directional'' divergence bound, the same proof can be used to bound $D_\alpha(Q||P)$.

For arbitrary $\delta_1, \delta_2$, plugging in $r = cL(\sqrt{d} + \sqrt{\ln(T/\delta_1)}+ \sqrt{\ln(T/\delta_2)})\sqrt{\eta}$ into Corollary~\ref{cor:smallstepconvergence} (where $c$ is the constant specified in Lemma~\ref{lemma:radiustailbound}) and using the definition $T = \tau/\eta$ we get that 
$$D_{\alpha'}(X_{T\eta}|| X_{T\eta}') \leq  \frac{3\tau {\alpha'} L^4 c^2 (d + \ln(\frac{\tau}{\eta \delta_1}) + \ln(\frac{\tau}{\eta \delta_2})) \eta}{4}$$ 

for all $k \in \mathbb{Z}^+$ and $X_{T\eta}, X_{T\eta}'$ as defined in Corollary~\ref{cor:smallstepconvergence}. Using the definition of R\'enyi divergence, this gives:

$$\int_{\mathbb{R}^d } \frac{X_{T\eta}(x)^{\alpha'}}{ X'_{T\eta}(x)^{{\alpha'} - 1}} \D{x} \leq \int_{\mathbb{R}^d } \frac{X_{T\eta}(x)^{\alpha'}}{ X'_{T\eta}(x)^{{\alpha'} - 1}}\D{x} +\frac{\Pr_{x \sim X_{T\eta}}[x = \bot]^{\alpha'}}{ \Pr_{x \sim X'_{T\eta}}[x = \bot]^{{\alpha'} - 1}}  \leq \frac{c_1({\alpha'})}{{\delta_1}^{c_2({\alpha'})}\delta_2^{c_3({\alpha'})}}, $$

where:

$$c_1({\alpha'}) = \exp\left(\frac{3\tau {\alpha'}({\alpha'} - 1) L^4 c^2 (d + 2\ln(\frac{\tau}{\eta})) \eta}{4}\right),$$
$$c_2({\alpha'}) = c_3({\alpha'}) = \frac{3\tau{\alpha'}({\alpha'}-1)L^4c^2\eta}{4}.$$

\textbf{Removing the conditioning on the continuous chain: }Let $\cond_{\delta_1}$ denote the (at least probability $1-\delta_1$) event that the conditions in Lemma~\ref{lemma:radiustailbound} are satisfied for the discrete chain and $\cond_{\delta_2}$ denote the (at least probability $1-\delta_2$) event that the conditions in Lemma~\ref{lemma:radiustailbound} are satisfied for the continuous chain. By Lemma~\ref{lemma:radiustailbound}, we have $Q(x) \geq X'_{T\eta}(x), Q(x | \cond_{\delta_2}) \leq \frac{1}{1 - \delta_2} X'_{T\eta}(x)$. Then for $\delta_2 < 1/2$:

\begin{align*}
\ex_{x \sim Q} \left[\frac{X_{T\eta}(x)^{\alpha'}}{ Q(x)^{{\alpha'}}} \biggr\vert \cond_{\delta_2} \right] &=
\int_{\mathbb{R}^d } Q(x | \cond_{\delta_2}) \frac{X_{T\eta}(x)^{\alpha'}}{ Q(x)^{\alpha'}} \D{x}\\
&\leq \frac{1}{1 - \delta_2}\int_{\mathbb{R}^d } \frac{X_{T\eta}(x)^{\alpha'}}{ X'_{T\eta}(x)^{{\alpha'} - 1}} \D{x} \\
&\leq \frac{2 \cdot c_1({\alpha'})}{{\delta_1}^{c_2({\alpha'})}\delta_2^{c_3({\alpha'})}}.
\end{align*}

This statement holds independent of $\delta_2$. We will eventually choose $\alpha'$ such that for the choice of $\eta$ specified in the lemma statement, $c_1(\alpha') < 2, c_3({\alpha'}) < 1$. Then applying Lemma~\ref{lemma:expectationfromconditional} with $Y = \frac{X_{T\eta}(x)^{{\alpha'}/2}}{ Q(x)^{{\alpha'}/2}}$ $\theta = 2$, $\beta = \frac{2c_1({\alpha'})}{\delta_1^{c_2({\alpha'})}}$, $\gamma = c_3({\alpha'})$, we get:

$$\ex_{x \sim Q}\left[\frac{X_{T\eta}(x)^{{\alpha'}/2}}{ Q(x)^{{\alpha'}/2}}\right] \leq \frac{8}{\delta_1^{c_2(\alpha')/2}}.$$

\textbf{Removing the conditioning on the discrete chain: }We now turn to removing the conditioning on $\cond_{\delta_1}$. Here we need to be a bit more careful since unlike with $X'_{T\eta}(x)$, $X_{T\eta}(x)$ is in the numerator and so the inequality $X_{T\eta}(x) \leq P(x)$ is facing the wrong way. Since $P, Q$ have the same support, we note that:

\begin{align*}
\ex_{x \sim Q}\left[\frac{X_{T\eta}(x)^{{\alpha'}/2}}{ Q(x)^{{\alpha'}/2}}\right] &= \ex_{x \sim P}\left[\frac{X_{T\eta}(x)^{{\alpha'}/2-1}}{ Q(x)^{{\alpha'}/2-1}} \cdot \frac{X_{T \eta}(x)}{P(x)} \right]\\
&\stackrel{(\star)}{=} \frac{\alpha'}{2}\ex_{x \sim P, y \sim Unif(0, P(x))}\left[ \frac{ y^{\alpha'/2-1}}{Q(x)^{\alpha'/2-1}} \cdot \mathbb{I}\left[y \leq X_{T\eta}(x)\right] \right]\\
&= \frac{\alpha'}{2} \ex_{x \sim P, y \sim Unif(0, P(x))}\left[ \frac{ y^{\alpha'/2-1}}{Q(x)^{\alpha'/2-1}} \biggr\vert y \leq  X_{T\eta}(x) \right]\\
& \qquad \qquad \cdot \Pr_{x \sim P, y \sim Unif(0, P(x))}\left[y \leq X_{T\eta}(x) \right]\\
&= \frac{\alpha'}{2}\ex_{x \sim P, y \sim Unif(0, P(x))}\left[ \frac{ y^{\alpha'/2-1}}{Q(x)^{\alpha'/2-1}} \biggr\vert \cond_{\delta_1}\right] \cdot (1-\delta_1).
\end{align*}

$(\star)$ follows as for any given $x$, we have:

\begin{align*}
X_{T\eta}(x)^{{\alpha'}/2-1} &= \frac{1}{X_{T\eta}(x)}X_{T\eta}(x)^{\alpha'/2}\\
&= \int_0^{X_{T\eta}(x)} \frac{1}{X_{T\eta}(x)} \frac{\alpha'}{2} y^{\alpha'/2 - 1} \D{y}\\
&= \frac{P(x)}{X_{T\eta}(x)} \int_0^{X_{T\eta}(x)} \frac{1}{P(x)} \frac{\alpha'}{2} y^{\alpha'/2 - 1} \D{y}\\
&= \frac{P(x)}{X_{T\eta}(x)} \int_0^{P(x)} \frac{1}{P(x)} \frac{\alpha'}{2} y^{\alpha'/2 - 1} \cdot \mathbb{I}\left[y \leq X_{T\eta}(x)\right] \D{y}\\
&= \frac{P(x)}{X_{T\eta}(x)} \frac{\alpha'}{2} \ex_{y \sim Unif(0, P(x))}\left[y^{\alpha'/2 - 1} \cdot \mathbb{I}\left[y \leq X_{T\eta}(x)\right]\right].\\
\end{align*}

% \begin{align*}
% \ex_{x \sim Q}\left[\frac{X_{T\eta}(x)^{{\alpha'}/2}}{ Q(x)^{{\alpha'}/2}}\right] &= \ex_{x \sim P}\left[\frac{X_{T\eta}(x)^{{\alpha'}/2-1}}{ Q(x)^{{\alpha'}/2-1}}\right]\\
% &\stackrel{(\star)}{\geq} \frac{\alpha'}{2}\ex_{x \sim P, y \sim Unif(0, P(x))}\left[ \frac{ y^{\alpha'/2-1}}{Q(x)^{\alpha'/2-1}} \cdot \mathbb{I}\left[y \leq X_{T\eta}(x)\right] \right]\\
% &= \frac{\alpha'}{2} \ex_{x \sim P, y \sim Unif(0, P(x))}\left[ \frac{ y^{\alpha'/2-1}}{Q(x)^{\alpha'/2-1}} \biggr\vert y \leq  X_{T\eta}(x) \right]\\
% & \qquad \qquad \cdot \Pr_{x \sim P, y \sim Unif(0, P(x))}\left[y \leq X_{T\eta}(x) \right]\\
% &\geq \frac{\alpha'}{2}\ex_{x \sim P, y \sim Unif(0, P(x))}\left[ \frac{ y^{\alpha'/2-1}}{Q(x)^{\alpha'/2-1}} \biggr\vert \cond_{\delta_1}\right] \cdot (1-\delta_1).
% \end{align*}

% $(\star)$ follows as for any given any $x$, we have:

% \begin{align*}
% X_{T\eta}(x)^{{\alpha'}/2-1} &= \frac{1}{X_{T\eta}(x)}X_{T\eta}(x)^{\alpha'/2}\\
% &= \int_0^{X_{T\eta}(x)} \frac{1}{X_{T\eta}(x)} \frac{\alpha'}{2} y^{\alpha'/2 - 1} \D{y}\\
% &\geq \int_0^{X_{T\eta}(x)} \frac{1}{P(x)} \frac{\alpha'}{2} y^{\alpha'/2 - 1} \D{y}\\
% &= \int_0^{P(x)} \frac{1}{P(x)} \frac{\alpha'}{2} y^{\alpha'/2 - 1} \cdot \mathbb{I}\left[y \leq X_{T\eta}(x)\right] \D{y}\\
% &= \frac{\alpha'}{2} \ex_{y \sim Unif(0, P(x))}\left[y^{\alpha'/2 - 1} \cdot \mathbb{I}\left[y \leq X_{T\eta}(x)\right]\right].\\
% \end{align*}

In turn, for all $\delta_1 < 1/2$, we have
$$\ex_{x \sim P, y \sim Unif(0, P(x))}\left[ \frac{ y^{\alpha'/2-1}}{Q(x)^{\alpha'/2-1}} \biggr\vert \cond_{\delta_1}\right] \leq \frac{32}{\alpha' \delta_1^{c_2(\alpha')/2}}.$$

If $c_2(\alpha')/2 < 1/2$ (which is equivalent to $c_2(\alpha') = c_3(\alpha') < 1$), by applying Lemma~\ref{lemma:expectationfromconditional} for $\theta = 2$ with $X = \frac{y^{\alpha'/4 - 1/2}}{Q(x)^{\alpha'/4 - 1/2}}, \beta = \frac{32}{\alpha'}, \gamma = c_2(\alpha')/2$ we get:

$$\ex_{x \sim P, y \sim Unif(0, P(x))}\left[\frac{y^{\alpha'/4 - 1/2}}{Q(x)^{\alpha'/4 - 1/2}}\right] \leq \frac{19}{\sqrt{\alpha'}} \implies$$
\begin{align*}
\ex_{x\sim Q}\left[\frac{P(x)^{\alpha'/4+1/2}}{Q(x)^{\alpha'/4+1/2}}\right] &= \left(\frac{\alpha'}{4} + \frac{1}{2}\right) \ex_{x \sim P, y \sim Unif(0, P(x))}\left[\frac{y^{\alpha'/4 - 1/2}}{Q(x)^{\alpha'/4 - 1/2}}\right] \\
&\leq \frac{19(\alpha'/4 + 1/2)}{\sqrt{\alpha'}} \\
&\leq 15 \sqrt{\alpha'}.
\end{align*}

\textbf{From moderate $\alpha'$-R\'enyi divergence to small $\alpha$-R\'enyi divergence: } If $\eps \geq \frac{3 \ln \alpha}{\alpha - 1}$, without loss of generality we can assume e.g. $\alpha \geq 4$ (by monotonocity of R\'enyi divergences, if $\alpha < 4$ it suffices to bound the $4$-R\'enyi divergence instead of the $\alpha$-R\'enyi divergence at the loss of a constant in the bound for $\eta$). Then for $\alpha' = 4\alpha - 2$ the preceding inequality lets us conclude the lemma holds. Otherwise, for $1 < \kappa < \alpha' / 4 + 1/2$, for $\alpha = \frac{\alpha' / 4 + 1/2}{\kappa}$, by Jensen's inequality we get:

$$\frac{1}{\alpha - 1} \ln \ex_{x\sim Q}\left[\frac{P(x)^\alpha}{Q(x)^\alpha}\right]\leq\frac{1}{\alpha - 1} \ln \left( \ex_{x\sim Q}\left[\frac{P(x)^{\alpha\kappa}}{Q(x)^{\alpha\kappa}}\right]^{1/\kappa}\right) \leq \frac{\ln 15 + \frac{1}{2}\ln \alpha + \frac{1}{2}\ln \kappa}{(\alpha - 1)\kappa}.$$ 

Choosing $\kappa = \frac{3 \ln \alpha \cdot \ln 1/\eps}{(\alpha - 1) \eps}$ then gives $D_\alpha(P || Q) \leq \eps$ as desired (note that for $\eps < \frac{3 \ln \alpha}{\alpha - 1}$ we have $\kappa > 1$ as is required). Now, we just need to verify that $c_1(\alpha') < 2, c_2(\alpha') = c_3(\alpha') < 1$ holds for $\alpha' = \frac{12 \alpha \ln \alpha \cdot \ln 1/\eps}{(\alpha - 1) \eps} - 2$. Since $c_2(\alpha') = c_3(\alpha') < \ln(c_1(\alpha')) / d$, it just suffices to show $c_1(\alpha') < 2$. This holds if:

$$\frac{3\tau {\alpha'}({\alpha'} - 1) L^4 c^2 (d + 2\ln(\frac{\tau}{\eta})) \eta}{4} < \ln 2,$$

which is given by choosing $\eta = \tilde{O}(\frac{1}{\tau L^4 \ln^2 \alpha} \cdot \frac{\eps^2}{d})$ with a sufficiently small constant hidden in $\tilde{O}$.
\end{proof}

We now apply results from \cite{VempalaW19} and the weak triangle inequality for R\'enyi divergence to get a bound on the number of iterations of discrete overdamped Langevin dynamics needed to achieve $\alpha$-R\'enyi divergence $\eps$:

\begin{lemma}\label{lemma:startingdivergence}
If $R(x) = e^{-f(x)}$ is a probability distribution over $\mathbb{R}^d$ with stationary point $0$ and $f$ is $1$-strongly convex and $L$-smooth, then for all $\alpha \geq 1$ we have:
$$D_\alpha\left(N\left(0, \frac{1}{L}I_d\right) || R\right) \leq \frac{d}{2} \ln L.$$
\end{lemma}
\begin{proof}
This follows from Lemma 4 in \cite{VempalaW19}, which gives the bound $D_\alpha(N(0, \frac{1}{L}I_d) || R) \leq f(\bzero) + \frac{d}{2} \ln \frac{L}{2\pi}$. We then note that the $1$-strongly convex, $L$-smooth $f$ with the maximum $f(\bzero)$ is given when $R$ is $N(0, I_d)$, which has density $R(x) = e^{-\left(\frac{d}{2} \ln (2 \pi) + \frac{1}{2}x^\top x \right)}$.
\end{proof}

It is well-known that $1$-strong convexity of $f$ implies that $p \propto e^{-f}$ satisfies log-Sobolev inequality with constant $1$ (see e.g. \cite{BakryE1985}). We then get:

\begin{lemma}[Theorem 2, \cite{VempalaW19}]\label{lemma:divergencedecay}
Fix any $f$ that is $1$-strong\-ly convex. Let $Q_t$ be the distribution arrived at by running overdamped Langevin dynamics using $f$ for continuous time $t$ from initial distribution $Q_0$. Then for the distribution $R$ satisfying $R(x) \propto e^{-f(x)}$ and any $\alpha \geq 1$:
$$D_\alpha(Q_t || R) \leq e^{-2t/\alpha} D_\alpha(Q_0|| R).$$
\end{lemma}

\begin{proof}[Proof of Theorem~\ref{thm:mainthm}]
We will prove the bound for $\alpha \geq 3/2$ - the bound for $1 \leq \alpha < 3/2$ follows by just applying monotonicity to the bound for $\alpha = 3/2$, at the loss of a multiplicative constant on $\tau, \eta$, and the iteration complexity.

Let $R$ be the distribution arrived at by running continuous overdamped Langevin dynamics using $f$ for time $\tau$ from initial distribution $N(0, \frac{1}{L}I_d)$. $N(0, \frac{1}{L}I_d)$ satisfies \eqref{eq:startingbound}, so from Theorem~\ref{lemma:conditionaldivergence} we have $D_{2\alpha} (P || Q) \leq \eps / 3$. From Lemmas~\ref{lemma:startingdivergence} and~\ref{lemma:divergencedecay} we have $D_{2\alpha}(Q || R) \leq \eps / 3$. Then, we use the weak triangle inequality of R\'enyi divergence (Fact~\ref{fact:triangleineq}) with $p, q = 2$ to conclude that $D_\alpha(P || R) \leq \eps$.
\end{proof}
\subsection{Langevin Dynamics with Bounded Gradients}

With only a minor modification to the analysis of the strongly convex and smooth case, we can also give a discretization error bound when $f$ is $B$-Lipschitz instead of strongly convex (while still $L$-smooth). We have the following radius tail bound analogous to Lemma~\ref{lemma:radiustailbound}:

\begin{lemma}\label{lemma:radiustailbound-lipschitz}
For all $\eta \leq 1$ and any $B$-Lipschitz, $L$-smooth $f$, let $x_t$ be the random variable given by running the discretized overdamped Langevin dynamics starting from an arbitrary initial distribution for continuous time $t$. Then with probability $1 - \delta$ over $\{x_t: t \in [0, T\eta]\}$, for all $t \leq T\eta$ and for a sufficiently large constant $c$:

$$\norm{x_t - x_{\lfloor t / \eta \rfloor \eta}}_2 \leq c (B + \sqrt{d} + \sqrt{\ln(T/\delta)})\sqrt{\eta}).$$

Similarly, if $x_t'$ is the random variable given by running continuous overdamped Langevin dynamics starting from an arbitrary initial distribution for time $t$, with probability $1-\delta$ over $x_t'$ for all $t \leq T\eta$:

$$ \norm{x'_t - x'_{\lfloor t / \eta \rfloor \eta}}_2 \leq c (B + \sqrt{d} + \sqrt{\ln(T/\delta)})\sqrt{\eta}).$$
\end{lemma}

The proof is deferred to Section~\ref{section:tailbounds}. This gives:

\begin{theorem}\label{lemma:conditionaldivergence-lipschitz}
For any $B$-Lipschitz, $L$-smooth function $f$, let $P$ be the distribution of states for discretized overdamped Langevin dynamics with step size $\eta$ and $Q$ be the distribution of states for continuous overdamped Langevin dynamics, both run from arbitrary initial distribution for continuous time $\tau$ that is a multiple of $\eta$. Then for $\alpha > 1$, $\eps > 0$, if $\eta = \tilde{O}(\frac{1}{\tau L^4 \ln^2 \alpha} \cdot \frac{\eps^2}{B^2 + d})$ we have $D_\alpha(P || Q), D_\alpha(Q || P)  \leq \eps$.
\end{theorem}

The proof of Theorem~\ref{lemma:conditionaldivergence-lipschitz} follows identically to Theorem~\ref{lemma:conditionaldivergence}, except using Lemma~\ref{lemma:radiustailbound-lipschitz} instead of Lemma~\ref{lemma:radiustailbound}.
% \begin{theorem}\label{lemma:conditionaldivergence-lipschitz}
% For any $B$-Lipschitz, $L$-smooth stationary distribution proportional to $e^{-f}$, let $P$ be the distribution of states for discretized overdamped Langevin dynamics with timestep $\eta$ and $Q$ be the distribution of states for continuous overdamped Langevin dynamics, both run from arbitrary initial distribution for continuous time $\tau$ that is a multiple of $\eta$. Then for $\alpha > 1$, $\eps > 0$, if $\eta = \tilde{O}(\frac{1}{\tau L^4 \ln^2 \alpha} \cdot \frac{\eps^2}{B^2 + d})$ we have $D_\alpha(P || Q) \leq \eps$.
% \end{theorem}
\section{Making The Bound Bi-Directional}\label{section:bidirectional}
In this section, we show that with slight modifications to the proof of Theorem~\ref{thm:mainthm}, $D_\alpha(P||R)$ and $D_\alpha(R||P)$ can be simultaneously bounded, proving Theorem~\ref{thm:maindpthm}. 

Note that Theorem~\ref{lemma:conditionaldivergence} provides bounds on both $D_\alpha(P||Q)$ and $D_\alpha(Q||P)$ for $Q$ that is the finite time distribution of the continuous chain. So, we just need to show that the following claim holds: for an appropriate choice of initial distribution, $D_\alpha(Q||R), D_\alpha(R||Q)$  are both small after sufficiently many iterations. To show this claim, we use the following results, all of which are slight modifications of the results in \cite{VempalaW19}. For completeness, we provide the proofs of these claims at the end of the section. We first need a lemma analogous to Lemma~\ref{lemma:divergencedecay} to show that $D_\alpha(R||Q)$ decays exponentially:

\begin{lemma}\label{lemma:divergencedecay2}
Fix any $f$ that is $1$-strongly convex. Let $Q_t$ be the distribution arrived at by running overdamped Langevin dynamics using $f$ for continuous time $t$ from initial distribution $Q_0$ such that $-\log Q_0$ is 1-strongly convex. Then for the distribution $R$ satisfying $R(x) \propto e^{-f(x)}$, any $\alpha > 1$, and any $t$:
$$D_\alpha(R || Q_t) \leq e^{-t/\alpha} D_\alpha(R|| Q_0).$$
\end{lemma}

This proof follows similarly to Lemma 2 in~\cite{VempalaW19}. If $D_\alpha(R||Q_0)$ and $D_(Q_0||R)$ were both initially not too large, Lemma~\ref{lemma:divergencedecay2} along with Lemma~\ref{lemma:divergencedecay} would be enough to arrive at Theorem~\ref{thm:maindpthm}. However, for any initial distribution $Q_0$, there is some $R$ satisfying the conditions of Lemma~\ref{lemma:divergencedecay2} such that for sufficiently large $\alpha$ one of $D_\alpha(R||Q_0)$ and $D_\alpha(Q_0||R)$ is infinite. The following hypercontractivity property of the Langevin dynamics gives that as long as $D_\alpha(Q_0||R)$ is finite for some small $\alpha$, it will become finite for larger $\alpha$ after a short amount of time:

\begin{lemma}[Lemma 14, \cite{VempalaW19}]\label{lemma:hypercontractivity}
Fix any $f$ that is $1$-strong\-ly convex. Let $Q_t$ be the distribution arrived at by running overdamped Langevin dynamics using $f$ for continuous time $t$ from initial distribution $Q_0$. Fix any $\alpha_0 > 1$, and let $\alpha_t = 1 + e^{2t}(\alpha_0 - 1)$. Then for the distribution $R$ satisfying $R(x) \propto e^{-f(x)}$:
$$D_{\alpha_t}(Q_t || R) \leq \frac{1 - 1/\alpha_0}{1-1/\alpha_t}  D_{\alpha_0}(Q_0|| R).$$
\end{lemma}

Given this lemma, we can now settle for an initial distribution where $D_\alpha(R||Q_0)$ is not too large for all $\alpha$, and $D_\alpha(Q_0||R)$ is not too large for $\alpha$ slightly larger than 1. Lemma~\ref{lemma:hypercontractivity} then says that $D_\alpha(Q_0||R)$ will be eventually be not too large after time $O(\log \alpha)$, at which point we can apply Lemmas~\ref{lemma:divergencedecay} and~\ref{lemma:divergencedecay2}. We now just need to show that our choice of initial distribution $N(0, I_d)$ satisfies these conditions:

\begin{lemma}\label{lemma:privateinitialdistribution}
Let $Q_0 = N(0, I_d)$. If $R(x) = e^{-f(x)}$ is a probability distribution over $\mathbb{R}^d$ with stationary point $0$ and $f$ is 1-strongly convex and $L$-smooth, then for all $\alpha \geq 1$ we have:
\[D_\alpha(R || Q_0) \leq d \log L.\]
In addition:
\[D_{1+1/L}(Q_0 || R) \leq \frac{dL \log L}{2}.\]
\end{lemma}

Putting it all together, we can now prove Theorem~\ref{thm:maindpthm}.
\begin{proof}[Proof of Theorem~\ref{thm:maindpthm}]
Let $Q_t$ be the distribution of the continuous overdamped Langevin dynamics using $f$ run from initial distribution $N(0, I_d)$ for time $t$. Assume without loss of generality that $\alpha \geq 2$, since if $\alpha \leq 2$ we can use monotonicity of R\'enyi divergences to bound e.g. $D_\alpha(P||R)$ by $D_2(P||R)$. 

If $\tau$ is at least a sufficiently large constant times $\alpha \ln \frac{d \ln L}{\epsilon}$, Lemma~\ref{lemma:privateinitialdistribution} and Lemma~\ref{lemma:divergencedecay2} give that $D_{2\alpha}(R||Q_\tau) \leq \epsilon / 3$. Theorem~\ref{lemma:conditionaldivergence} gives that $D_{2\alpha}(Q_\tau||P) \leq \epsilon / 3$. Fact~\ref{fact:triangleineq} with $p, q = 2$ gives that $D_\alpha(R||P) \leq \epsilon$.

 Lemma~\ref{lemma:hypercontractivity} and Lemma~\ref{lemma:privateinitialdistribution} give that at time $t = \frac{1}{2} \log ((2\alpha - 1)L)$, $D_{2\alpha}(Q_t || R) \leq d \log L$. Then Lemma~\ref{lemma:divergencedecay} gives that, $D_{2\alpha}(Q_\tau || R) \leq \epsilon / 3$.  Theorem~\ref{lemma:conditionaldivergence} gives that $D_{2\alpha}(P||Q_\tau) \leq \epsilon / 3$. Fact~\ref{fact:triangleineq} with $p, q = 2$ again gives that $D_\alpha(P||R) \leq \epsilon$.
\end{proof}

\subsection{Proof of Lemma~\ref{lemma:divergencedecay2}}

To prove Lemma~\ref{lemma:divergencedecay2}, we modify the proofs of Lemma 4 and 5 of~\cite{VempalaW19}. To describe the modifications, we reintroduce the following definitions from that paper:

\begin{definition}
We say that a distribution $Q$ has LSI constant $\kappa$ if for all smooth functions $g:\mathbb{R}^n \rightarrow \mathbb{R}$ for which $\ex_{x \sim Q}[g(x)^2] < \infty$:

\[\ex_{x \sim Q}\left[g(x)^2 \log\left(g(x)^2\right)\right] - \ex_{x \sim Q}\left[g(x)^2\right] \log\left(\ex_{x \sim Q}\left[g(x)^2\right]\right) \leq \frac{2}{\kappa} \ex_{x \sim Q}\left[\norm{\grad{g(x)}}^2\right].\]
\end{definition}

\begin{definition}
We define for $\alpha \neq 0, 1$:

$$F_\alpha(Q||R) = \ex_{x \sim R}\left[ \frac{Q(x)^\alpha}{R(x)^{\alpha}}\right],$$
$$G_\alpha(Q||R) = \ex_{x \sim R}\left[\frac{Q(x)^\alpha}{R(x)^{\alpha}} \norm{\grad \log \frac{Q(x)}{R(x)}}_2^2\right] = \frac{4}{\alpha^2} \ex_{x \sim R} \left[\norm{\grad \left(\frac{Q(x)}{R(x)}\right)^{\alpha/2}}_2^2\right].$$

For $\alpha = 0, 1$ these quantities are defined as their limit as $\alpha$ goes to $0, 1$ respectively.
\end{definition}
Unlike~\cite{VempalaW19}, we extend this definition to negative values of $\alpha$, which allows us to swap the arguments $Q, R$:

\begin{fact}\label{fact:swap}
$F_{1-\alpha}(Q||R) = F_{\alpha}(R||Q), G_{1 - \alpha}(Q||R) = G_{\alpha}(R||Q).$ 
We also recall that $D_{1-\alpha}(Q||R) = \frac{1-\alpha}{\alpha} D_{\alpha}(R||Q)$.
\end{fact}

\begin{proof}[Proof of Lemma~\ref{lemma:divergencedecay2}]
\cite{BakryE1985} shows that since the initial distribution satisfies that $-\log Q_0$ is 1-strongly convex, $Q_0$ has LSI constant 1. Consider instead running the discrete overdamped Langevin dynamics with step size $\eta$ starting with $Q_0$. In one step, we apply a gradient descent step that is $(1-\eta/2)$-Lipschitz (see e.g. \citep[Lemma 3.7]{HardtRS16}), and then add Gaussian noise $N(0, 2\eta I_d)$. Lemma 16 in~\cite{VempalaW19} shows that applying a $(1-\eta/2)$-Lipschitz map to a distribution with LSI constant $c$ results in a distribution with LSI constant at least $c/(1-\eta/2)^2$. Adding Gaussian noise $N(0, 2\eta I_d)$ to a distribution with LSI constant $c$ results in a distribution with LSI constant at least $\frac{1}{1/c + 2\eta}$ (see e.g.~\citep[Proposition 1.1]{WangW16}). Putting it together, we get that after one step of the discrete dynamics, the LSI constant of the distribution goes from $c$ to at least:

$$\frac{1}{\frac{(1-\eta/2)^2}{c}+2\eta} = \frac{c}{1 - (1 - 2c)\eta + \eta^2/4}.$$

Then, we have that $1 - (1 - 2c)\eta + \eta^2/4 \leq 1$, i.e. the LSI constant does not decrease after one step, as long as $\eta \leq 4(1 - 2c)$. Taking the limit as $\eta$ goes to 0, we conclude that $Q_t$'s LSI constant can never decrease past 1/2, i.e. $Q_t$ has LSI constant at least 1/2 for all $t \geq 0$.

Now, since $Q_t$ has LSI constant at least 1/2, we can repeat the proof of Lemma 5 in \cite{VempalaW19} with the distributions swapped to show that $\frac{G_{\alpha}(R||Q_t)}{F_{\alpha}(R||Q_t)} \geq \frac{1}{\alpha^2} D_\alpha(R||Q_t)$. Applying Fact~\ref{fact:swap} to the proof of Lemma 6 in \cite{VempalaW19}, we can show that $\frac{\D{}}{\D{t}} D_\alpha(R||Q_t) = -\alpha \frac{G_{\alpha}(R||Q_t)}{F_{\alpha}(R||Q_t)}$. Combining these two inequalities and integrating gives the lemma.
\end{proof}

\subsection{Proof of Lemma~\ref{lemma:privateinitialdistribution}}

The proof of Lemma~\ref{lemma:privateinitialdistribution} follows similarly to that of Lemma~\ref{lemma:startingdivergence}.

\begin{proof}[Proof of Lemma~\ref{lemma:privateinitialdistribution}]
Since $f$ is 1-strongly convex and $L$-smooth, we have:

$$f(\bzero)+\frac{1}{2}\norm{x}_2^2 \leq f(x) \leq f(\bzero) + \frac{L}{2}\norm{x}_2^2.$$

Then:

\begin{align*}
\exp((\alpha - 1) D_\alpha(R||Q_0)) &= \int_{\mathbb{R}^d} \frac{R(x)^\alpha}{Q_0(x)^{\alpha-1}} \D{x}\\
&= (2\pi)^{d(\alpha-1)/2}\int_{\mathbb{R}^d} \exp\left(-\alpha f(x) + \frac{\alpha-1}{2}\norm{x}_2^2\right) \D{x}\\
&\leq \frac{(2\pi)^{d(\alpha-1)/2}}{e^{\alpha f(\bzero)}}\int_{\mathbb{R}^d} \exp\left(-\frac{1}{2}\norm{x}_2^2\right) \D{x}\\
&=\frac{(2\pi)^{d\alpha/2}}{e^{\alpha f(\bzero)}}. 
\end{align*}

Taking logs and using that the $L$-smooth $f$ that minimizes $f(\bzero)$ is $N(0, \frac{1}{L}I_d)$ with density $\exp(-\frac{d}{2} \log (2\pi / L) - L \norm{x}_2^2)$:

$$D_\alpha(R||Q_0) \leq \frac{\alpha}{\alpha-1} \cdot \left(\frac{d}{2} \log 2\pi - f(\bzero)\right) \leq \frac{\alpha}{\alpha - 1} \cdot \frac{d}{2} \log L.$$

 For $\alpha \geq 2$, the above bound is thus at most $d \log L$ as desired, and for $1 \leq \alpha \leq 2$ we can just use monotonicity of R\'enyi divergences to bound $D_\alpha(R||Q_0)$ by $D_2(R||Q_0)$.

Similarly:

\begin{align*}
\exp((1/L) D_{1+1/L}(Q_0||R)) &= \int_{\mathbb{R}^d} \frac{Q_0(x)^{1+1/L}}{R(x)^{1/L}} \D{x}\\
&= (2\pi)^{-d(1+1/L)/2}\int_{\mathbb{R}^d} \exp\left(-\frac{1+1/L}{2}\norm{x}_2^2 + f(x)/L \right) \D{x}\\
&\leq \frac{e^{f(\bzero)/L}}{(2\pi)^{d(1+1/L)/2}}\int_{\mathbb{R}^d} \exp\left(-\frac{1}{2L}\norm{x}_2^2\right) \D{x}\\
&= \frac{e^{f(\bzero)/L} L^{d/2}}{(2\pi)^{d/2L}}.
\end{align*}

Taking logs, and using that the $1$-strongly convex $f$ that maximizes $f(\bzero)$ is $N(0, I_d)$ with density $\exp(-\frac{d}{2} \log (2\pi) - L \norm{x}_2^2)$:
$$D_{1+1/L}(Q_0 || R) \leq L \left[f(\bzero)/L + \frac{d}{2} \log L - \frac{d}{2L} \log (2\pi)\right] \leq \frac{dL \log L}{2}.$$

\end{proof}

\section{Underdamped Langevin Dynamics}\label{sec:underdamped}

Our approach can also be used to show a bound on the discretization error of \textit{underdamped} Langevin dynamics. We again start by bounding the divergence between two discrete processes with step sizes $\eta$ and $\eta/k$, whose limits as $k$ goes to infinity are the discretized and continuous underdamped Langevin dynamics. 
% To show the versatility of the analysis, we also show a bound on the discretization error of the SDE \eqref{eq:discreteud} in using the \textit{underdamped} Langevin dynamics to sample from $p(x) \propto e^{-f(x)}$. To bound the error due to discretization, we again start by bounding the divergence between two discrete processes with timesteps $\eta$ and $\eta/k$, whose limits as $k$ goes to infinity are the discretized and continuous underdamped Langevin dynamics. 
Again let $x_t$ denote the position of the chain using step size $\eta$ at continuous time $t$, and $x_t'$ denote the position of the chain using step size $\eta/k$. Let $v_t, v_t'$ denote the same but for velocity instead of position. If e.g. for the first chain we ever have $\norm{x_{t^*} - x_{\lfloor t^* / \eta \rfloor \eta}}_2 > r$ we will let $(x_t, v_t)$ equal $\bot$ for all $t \geq t^*$. We want to bound the divergence between the distributions $X_{0:Tk}$ over $\{(x_{i\eta/k}, v_{i\eta/k})\}_{0 \leq i \leq Tk}$ and $X'_{0:Tk}$ over $\{(x_{i\eta/k}', v_{i\eta/k}')\}_{0 \leq i \leq Tk}$. A sample from $X_{0:Tk}$ or $X'_{0:Tk}$ can be constructed by applying the following operations $Tk$ times to $\{(x_0, v_0)\}$ sampled from an initial distribution $X_0$:

\begin{itemize}
    \item To construct a sample from $X_{0:Tk}$, given a sample $\{(x_{i\eta/k}, v_{i\eta/k})\}_{0 \leq i \leq j}$ from $X_{0:j}$:
    \begin{itemize}
        \item If $(x_{j\eta/k}, v_{j\eta/k}) = \bot$ append $(x_{i\eta/k}, v_{i\eta/k}) = \bot$ to $\{(x_{i\eta/k}, v_{i\eta/k})\}_{0 \leq i \leq j}$.
        \item Otherwise, append $(x_{(j+1)\eta/k}, v_{(j+1)\eta/k})$ where:
        $$v_{(j+1)\eta/k} = (1 - \gamma \frac{\eta}{k})v_{j \eta /k} - \mu \frac{\eta}{k} \grad f(x_{\lfloor j/k \rfloor \eta}) + \xi_j,$$
        $$x_{(j+1)\eta/k} = x_{j \eta/k} + \frac{\eta}{k} v_{(j+1)\eta/k},$$
        and $\xi_j \sim N(0, 2\gamma \mu \frac{\eta}{k} I_d)$. Then if $\norm{x_{(j+1)\eta/k} - x_{\lfloor (j+1)/k\rfloor \eta}}_2 > r$ (i.e. $\cond_r$ no longer holds) replace $(x_{(j+1)\eta/k}, v_{(j+1)\eta/k})$ with $\bot$.
    \end{itemize}  
    Let $\psi$ denote this update, i.e. $X_{0:j+1} = \psi(X_{0:j})$.
    \item To construct a sample from $X'_{0:Tk}$, the update (which we denote $\psi'$) is identical to $\psi$ except we use the gradient at $x'_{j \eta / k}$ instead of $x'_{\lfloor j / k \rfloor \eta}$ to compute $v_{(j+1)\eta/k}$.
\end{itemize}

We remark that unlike in our analysis of the overdamped Langevin dynamics, for finite $k$, $X_{0:Tk}, X'_{0:Tk}$ do \textit{not} actually correspond to the SDE \eqref{eq:discreteud} with step size $\eta, \eta/k$. However, we still have the property that the limit of $X_{0:Tk}$ (resp. $X_{0:Tk}'$) as $k$ goes to infinity follows a discretized (resp. continuous) underdamped Langevin dynamics, which is all that is needed for our analysis. Similarly to the overdamped Langevin dynamics we have:

\begin{lemma}\label{lemma:smallstepconvergence-ud}
For any $L$-smooth $f$ and $X_{0:Tk}, X'_{0:Tk}$ as defined in Section~\ref{sec:underdamped}, we have:

$$D_\alpha(X_{0:Tk} || X'_{0:Tk}), D_\alpha(X'_{0:Tk} || X_{0:Tk}) \leq \frac{T\alpha L^2 r^2 \eta}{4} \cdot \frac{\mu}{\gamma}.$$
\end{lemma}

The proof follows almost exactly as did the proof of Lemma~\ref{lemma:smallstepconvergence}: we note that the updates to position are deterministic, and so by Fact~\ref{fact:postprocessing} we just need to control the divergence between velocities, which can be done using the same analysis as in Lemma~\ref{lemma:smallstepconvergence}. The multiplicative factor of $\mu/\gamma$ appears because the ratio of the Gaussian's standard deviation in any direction to the gradient step's multiplier is $\sqrt{\gamma/\mu}$ times what it was in the overdamped Langevin dynamics. Next, similar to Lemma~\ref{lemma:radiustailbound}, we have the following tail bound on $r$:

\begin{lemma}\label{lemma:radiustailbound-ud}
Fix any $\gamma \geq 2$, and define

$$v_{\max} := c \sqrt{\gamma \mu}\left(\sqrt{\tau d} +  \sqrt{\ln(1/\delta)}\right).$$

Fix any $\eta \leq \frac{\gamma}{\mu L}$, and any distribution over $x_0, v_0$ satisfying that

\begin{equation}\label{eq:startingbound-ud}
\Pr\left[\mu f(x_0) + \frac{\norm{v_0}_2^2}{2} \leq \frac{1}{2} v_{\max}^2\right] \geq 1 - \delta,
\end{equation}

let $x_t, v_t$ be the random variable given by running the discretized underdamped Langevin dynamics starting from $x_0, v_0$ drawn from this distribution for time $t$. Then with probability $1 - \delta$ over $\{(x_t, v_t) : t \in [0, \tau]\}$, for all $t \leq \tau$ that are multiples of $\eta$ and for a sufficiently large constant $c$:

$$\norm{x_{t+\eta} - x_{t}}_2 \leq v_{\max} \eta.$$

Similarly, if $x_t$ is the random variable given by running continuous underdamped Langevin dynamics starting from $x_0, v_0$ drawn from this distribution for time $t$, with probability $1-\delta$ over $\{(x'_t, v'_t) : t \in [0, \tau]\}$ for all $t \leq \tau$:

$$ \norm{x_t - x_{\lfloor t / \eta \rfloor \eta}}_2 \leq v_{\max} \eta.$$
\end{lemma}

The proof is deferred to Section~\ref{section:tailbounds}. We note that the correct tail bound likely has a logarithmic dependence on $\tau$ and not a polynomial one. However, based on similar convergence bounds (e.g. \cite{VempalaW19, MaCCFBJ19}), we conjecture that the time $\tau$ needed for continuous underdamped Langevin dynamics to converge in R\'enyi divergence has a logarithmic dependence on $d, 1/\eps$. So, improving the dependence on $\tau$ in this tail bound will likely not improve the final iteration complexity's dependence on $d, 1/\eps$ by more than logarithmic factors. In addition, settling for a polynomial dependence on $\tau$ makes the proof rather straightforward. Putting it all together, we get:

\begin{theorem}\label{lemma:conditionaldivergence-ud}
For any $1$-strongly convex, $L$-smooth function $f$, let $P$ be the distribution of states for discretized underdamped Langevin dynamics with step size $\eta$ and $Q$ be the distribution of states for continuous underdamped Langevin dynamics, both run from any initial distribution on $x_0, v_0$ satisfying \eqref{eq:startingbound-ud}, for continuous time $\tau$ that is a multiple of $\eta$. Then for $\alpha > 1$, $\eps > 0$, if $\eta = \tilde{O}(\min\{\frac{1}{ L \tau \mu \ln \alpha } \cdot \frac{\eps}{\sqrt{d}}, \frac{\gamma}{\mu L}\})$ we have $D_\alpha(P || Q), D_\alpha(Q || P)\leq \eps$.
\end{theorem}

\begin{proof}
The proof follows similarly to that of Theorem~\ref{lemma:conditionaldivergence}. From Lemma~\ref{lemma:smallstepconvergence-ud}, plugging in the tail bound of Lemma~\ref{lemma:radiustailbound-ud} for $r$ (which holds since we assume $\eta \leq \frac{\gamma}{\mu L}$) we get the divergence bound:
$$D_{\alpha'}(X_{T, k}, X_{T, k}') \leq  \frac{3 \mu \tau {\alpha'} L^2 c^2 ( \tau d + \ln(\frac{1}{\delta_1}) + \ln(\frac{1}{\delta_2})) \eta^2}{4}$$ 

We can then just follow the proof of Theorem~\ref{lemma:conditionaldivergence} as long as:

$$c_1({\alpha'}) = \exp\left(\frac{3 \mu \tau^2 d {\alpha'} (\alpha' - 1) L^2 c^2 \eta^2}{4}\right) < 2,$$

For $\alpha' = \frac{12 \alpha \ln \alpha \ln 1 / \epsilon}{(\alpha - 1)\epsilon} - 2$. This follows if $\eta = \tilde{O}(\frac{1}{ L \tau \mu \ln \alpha } \cdot \frac{\epsilon}{\sqrt{d}})$ as assumed in the lemma statement.
\end{proof}
% \begin{theorem}\label{lemma:conditionaldivergence-ud}
% For any $1$-strongly convex, $L$-smooth stationary distribution proportional to $e^{-f}$, let $P$ be the distribution of states for discretized underdamped Langevin dynamics with timestep $\eta$ and $Q$ be the distribution of states for continuous underdamped Langevin dynamics, both run from any initial distribution on $x_0, v_0$ satisfying \eqref{eq:startingbound-ud}, for continuous time $\tau$ that is a multiple of $\eta$. Then for $\alpha > 1$, $\eps > 0$, if $\eta = \tilde{O}(\min\{\frac{1}{ L \tau \mu \ln \alpha } \cdot \frac{\eps}{\sqrt{d}}, \frac{\gamma}{\mu L}\})$we have $D_\alpha(P || Q) \leq \eps$.
%\end{theorem}

We give here some intuition for why the proof achieves an iteration complexity for underdamped Langevin dynamics with a quadratically improved dependence on $d, \eps$ compared to overdamped Langevin dynamics. The tail bound on the maximum movement within each step of size $\eta$ (and in turn the norm of the discretization error due to the gradient) has a quadratically stronger dependence on $\eta$ in the underdamped case than in the overdamped case. In turn, in underdamped Langevin dynamics the ``privacy loss'' of hiding this error with Brownian motion also improves quadratically as a function of $\eta$.
\section{Proofs of Tail Bounds on Movement}\label{section:tailbounds}
In this section we give the proofs of Lemmas~\ref{lemma:radiustailbound},~\ref{lemma:radiustailbound-lipschitz}, and~\ref{lemma:radiustailbound-ud}, which provide tail bounds for the maximum movement within each step of the Langevin dynamics in the three settings we consider. We first recall some facts about Gaussians, Brownian motion, and gradient descent:

\begin{fact}[Univariate Gaussian Tail Bound]
For $X \sim N(0, \sigma^2)$ and any $x \geq 0$, we have
$$\Pr[X \geq x] = \Pr[X \leq -x] \leq \exp\left(- \frac{x^2}{2 \sigma^2}\right).$$
\end{fact}

\begin{fact}[Isotropic Multivariate Normal Tail Bound]
For $X \sim N(0, I_d)$ and any $x \geq 0$, we have
$$\Pr[\norm{X}_2 \geq \sqrt{d} + x] \leq \exp\left(-\frac{x^2}{2}\right).$$
\end{fact}

\begin{fact}[Univariate Brownian Motion Tail Bound]
Let $B_t$ be a standard (one-dim\-ensional) Brownian motion. For any $0 \leq a \leq b$, we have:

$$\Pr\left[\sup_{t \in [a, b]} [B_t - B_a] \geq x\right] = 2 \cdot \Pr[N(0, b-a) \geq x] \leq 2 \exp\left(-\frac{x^2}{2(b-a)}\right) $$
\end{fact}

The preceding fact is also known as \textit{the reflection principle}.

\begin{fact}[Multivariate Brownian Motion Tail Bound]
Let $B_t$ be a standard $d$-dim\-ensional Brownian motion. For any $0 \leq a \leq b$, we have:

$$\Pr\left[\sup_{t \in [a, b]} \norm{B_t - B_a}_2 \geq \sqrt{b-a} \left( \sqrt{d} + x \right) \right] \leq 2 \exp (-x^2 / 4). $$
\end{fact}

\begin{fact}[Discrete Gradient Descent Contracts]\label{fact:dgdc}
Let $f: \mathbb{R}^d \rightarrow \mathbb{R}$ be a 1-strongly convex, $L$-smooth function. Then for $\eta \leq \frac{2}{L+1}$, we have $\norm{x - \eta \grad f(x) - x' + \eta \grad f(x')}_2 \leq (1 - \frac{\eta L}{L+1}) \norm{x - x'}_2 \leq (1 - \frac{\eta}{2}) \norm{x - x'}_2$ for any $x, x' \in \mathbb{R}^d$.
\end{fact}
See e.g. \citep[Lemma 3.7]{HardtRS16} for a proof of this fact.

Since we assume $f$'s global minimum is at $0$ (and thus $\grad f(0) = 0$), as a corollary we have $\norm{x - \eta \grad f(x)}_2 \leq (1 - \eta/2) \norm{x}_2$. We also have as a corollary:

\begin{fact}[Continuous Gradient Descent Contracts]\label{fact:cgdc}
Let $f: \mathbb{R}^d \rightarrow \mathbb{R}$ be a 1-strongly convex, $L$-smooth function. Then for any $x_0, x_0' \in \mathbb{R}^d$ and $x_t, x_t'$ that are solutions to the differential equation $\D{x_t} = -\grad f(x_t) \D{t}$ we have $\norm{x_t - x_t'}_2 \leq e^{-t/2} \norm{x_0 - x_0'}_2$. 
\end{fact}
\begin{proof}
This follows by noting that the $x_t$ is the limit as integer $k$ goes to $\infty$ of applying $k$ discrete gradient descent steps to $x_0$ with $\eta = t/k$. So, the contractivity bound we get for $x_t$ is $\norm{x_t}_2 \leq \lim_{k \rightarrow \infty} (1 - t/2k)^k \norm{x_0}_2 = e^{-t/2} \norm{x_0}_2$.
\end{proof}

\subsection{Proof of Lemma~\ref{lemma:radiustailbound}}

\begin{proof}
We consider the discrete chain first. For each timestep starting at $t$ that is a multiple of $\eta$, using smoothness we have:

\begin{align*}
\max_{t' \in [t, t+ \eta)} \norm{x_{t'} - x_t}_2 &= \max_{t' \in [t, t+ \eta)} \norm{-(t' - t)\grad f(x_t) + \sqrt{2}\int_t^{t'} \D{B_s}}_2\\
&\leq \eta\norm{\grad f(x_t)}_2 +  \sqrt{2} \max_{t' \in [t, t+ \eta)} \norm{\int_t^{t'} \D{B_s}}_2\\
&\leq \eta L \norm{x_t}_2 +  \sqrt{2} \max_{t' \in [t, t+ \eta)} \norm{\int_t^{t'} \D{B_s}}_2.
\end{align*}

Using the tail bound for multivariate Brownian motion, $\max_{t' \in [t, t+ \eta)} \norm{\int_t^{t'} \D{B_s}}_2$ is at most $\frac{c}{2 \sqrt{2}}\left(\sqrt{d} + \sqrt{\ln(T/\delta)}\right)\sqrt{\eta}$ with probability at least $1-\frac{\delta}{2T}$ for each timestep. So it suffices to show that with probability at least $1 - \frac{\delta}{2}$, for all $0 \leq t < T\eta$ that are multiples of $\eta$, $\norm{x_t}_2 \leq \frac{c}{2\sqrt{\eta}}\left(\sqrt{d} + \sqrt{\ln(T/\delta)}\right)$. From \eqref{eq:startingbound}, with probability $1 - \frac{\delta}{T+1}$, $\norm{x_0}_2 \leq \frac{c}{2\sqrt{\eta}}\left(\sqrt{d} + \sqrt{\ln(T/\delta)}\right)$. We will show that if $\norm{x_t}_2 \leq \frac{c}{2\sqrt{\eta}}\left(\sqrt{d} + \sqrt{\ln(T/\delta)}\right)$ then with probability  $1-\frac{\delta}{T+1}$ we have $\norm{x_{t+\eta}}_2 \leq \frac{c}{2\sqrt{\eta}}\left(\sqrt{d} + \sqrt{\ln(T/\delta)}\right)$, completing the proof for the discrete case by a union bound. This follows because by Fact~\ref{fact:dgdc} the gradient descent step is $(1-\eta/2)$-Lipschitz for the range of $\eta$ we consider. This gives that after the gradient descent step but before adding Gaussian noise, $x_{t+\eta}$ has norm at most $(1-\eta/2)\norm{x_t}_2 \leq (1-\eta/2)\frac{c}{2\sqrt{\eta}}\left(\sqrt{d} + \sqrt{\ln(T/\delta)}\right)$. Then, $\norm{x_{t+\eta}}_2 > \frac{c}{2\sqrt{\eta}}\left(\sqrt{d} + \sqrt{\ln(T/\delta)}\right)$ only if $\sqrt{2} \norm{\int_t^{t+\eta} \D B_s}_2$ is larger than $c\sqrt{\eta}\left(\sqrt{d} + \sqrt{\ln(T/\delta)}\right)$, which happens with probability at most $\frac{\delta}{T+1}$ by the multivariate Gaussian tail bound.

We now consider the continuous chain. For all $t$ that are multiples of $\eta$:

\begin{align*}
\max_{u \in [t, t + \eta)}\norm{x'_{u} - x'_{t}}_2 &= \max_{u \in [t, t + \eta)}\norm{\int_{t}^u -\grad f(x'_s) \D{s} + \sqrt{2} \D{B_s}}_2\\
&\leq \eta L\max_{u \in [t, t + \eta)}\norm{x'_u}_2 + \max_{u \in [t, t + \eta)}\norm{\sqrt{2}\int_{t}^u \D{B_s}}_2.
\end{align*}

As with the discrete chain, the multivariate Brownian motion tail bound gives that

$$\max_{u \in [t, t + \eta)} \norm{\sqrt{2}\int_{t}^u \D{B_s}}_2 \leq \frac{c}{2}\left(\sqrt{d} + \sqrt{\ln(T/\delta)}\right)\sqrt{\eta},$$

with probability at least $1-\frac{\delta}{2T}$. 
So it suffices to show that at all times between 0 and $T \eta$, $\norm{x'_u}_2 \leq \frac{c}{2\sqrt{\eta}}\left(\sqrt{d} + \sqrt{\ln(T/\delta)}\right)$ with probability at least $1 - \frac{\delta}{2}$. 
We first claim that with probability at least $1 - \frac{\delta}{4}$, for all $t$ that are multiples of $\eta$, $\norm{x'_t}_2 \leq \frac{c}{4\sqrt{\eta}}\left(\sqrt{d} + \sqrt{\ln(T/\delta)}\right)$. 
This is true for $x'_0$ with probability at least $1 - \frac{\delta}{4(T+1)}$ by \eqref{eq:startingbound}. 
By contractivity of continuous gradient descent, $x'_{t + \eta}$ is equal to $Ax'_t + \sqrt{2}\int_t^{t+\eta} A_s' \D{B_s} $ for some $A$ which has eigenvalues in $[-e^{-\eta/2}, e^{- \eta/2}]$ and a set of matrices $\{A_s' | s \in [0, \eta]\}$ with eigenvalues in $[-e^{-(\eta-s)/2}, e^{-(\eta-s)/2}]$\footnote{In particular, recalling the proof of Facts~\ref{fact:dgdc} and~\ref{fact:cgdc}, we can write $A$ explicitly as $\lim_{k \rightarrow \infty} \prod_{j=0}^{k-1} (I_d - \frac{\eta}{k}\grad^2 f(z_j))$, where $z_j$ is some point on the path from $0$ to $x'_{t + \frac{j\eta}{k}}$. Each $A_s$ can be written similarly, except only considering the gradient descent process from time $t+s$ to $t+\eta$.}.
Then conditioning on the claim holding for $x'_t$, $\norm{x'_{t+\eta}}_2$ exceeds $\frac{c}{4\sqrt{\eta}}\left(\sqrt{d} + \sqrt{\ln(T/\delta)}\right)$ only if the norm of $\sqrt{2}\int_t^{t+\eta} A_s' \D{B_s}$ exceeds $\frac{c(1-e^{-\eta/2})}{4\sqrt{\eta}}\left(\sqrt{d} + \sqrt{\ln(T/\delta)}\right) \geq \frac{c(1-e^{-.5}))\sqrt{\eta} }{4}\left(\sqrt{d} + \sqrt{\ln(T/\delta)}\right)$. 
Since Brownian motion is rotationally symmetric, and all $A_s'$ have eigenvalues in $[-1, 1]$, this occurs with probability upper bounded by the probability $\sqrt{2}\int_t^{t+\eta} \D{B_s}$ exceeds this bound, which is at most $\frac{\delta}{4(T+1)}$ by the Brownian motion tail bound. 
The claim follows by taking a union bound over all $t$ that are multiples of $\eta$.

Then, conditioning on the event in the claim, for each corresponding interval $[t, t + \eta)$ since gradient descent contracts we have

\begin{align*}
\max_{u \in [t, t+\eta)} \norm{x'_u}_2 &\leq \norm{x'_t}_2 + \max_{u \in [t, t+\eta)} \norm{\sqrt{2}\int_{t}^u \D{B_s}}_2\\
&\leq \frac{c}{4\sqrt{\eta}}\left(\sqrt{d} + \sqrt{\ln(T/\delta)}\right) + \max_{u \in [t, t+\eta)} \norm{\sqrt{2}\int_{t}^u \D{B_s}}_2.
\end{align*}

We conclude by using the multivariate Brownian motion tail bound to observe that 

$$\max_{u \in [t, t+\eta)} \norm{\sqrt{2}\int_{t}^u \D{B_s}}_2 \leq \frac{c}{4\sqrt{\eta}}\left(\sqrt{d} + \sqrt{\ln(T/\delta)}\right),$$

with probability at least $1 - \frac{\delta}{4T}$, and then taking a union bound over all intervals.
\end{proof}

\subsection{Proof of Lemma~\ref{lemma:radiustailbound-lipschitz}}

\begin{proof}
By $B$-Lipschitzness of $f$, the movement in any interval of length $\eta$ due to the gradient step in both the discrete and continuous case is at most $2B\eta$. By the multivariate Brownian motion tail bound, in both the discrete and continuous cases the maximum movement due to the addition of Gaussian noise is at most $c(\sqrt{d} + \sqrt{\ln(T/\delta)})\sqrt{\eta}$ with probability at least $1 - \frac{\delta}{T}$ in each interval of length $\eta$, and then the lemma follows by a union bound and triangle inequality.
\end{proof}

\subsection{Proof of Lemma~\ref{lemma:radiustailbound-ud}}

\begin{proof}
We can assume $\delta < 1/2$, at a loss of a multiplicative constant. We first focus on the continuous chain. It suffices to show the maximum norm of the velocity over $[0, \tau)$ is $v_{\max}$ with the desired probability. We will instead focus on bounding the Hamiltonian, defined as follows:

$$\phi_t = \mu f(x'_t) + \norm{v'_t}_2^2 / 2.$$

Analyzing the rate of change, by Ito's lemma we get

\begin{align*}
\D{\phi_t} &= \frac{\del \phi_t}{\del x'_t} \cdot \D{x'_t} + \frac{\del \phi_t}{\del v'_t} \cdot \D{v'_t} + \frac{1}{2} \left[ \sum_{i, j \in [d]} \frac{\del^2 \phi_t}{\del (v'_t)_i \del (v'_t)_j} \frac{\D{(v'_t)_i}}{\D{B_t}} \frac{\D{(v'_t)_j}}{\D{B_t}} \right] \D{t}\\
&= \mu \grad f(x'_t) \cdot v'_t \D{t} + v'_t \cdot (- \mu \grad f(x'_t) \D{t} - \gamma v'_t \D{t} + \sqrt{2 \gamma \mu} \D{B_t}) + 2 \gamma \mu d \cdot \D{t}\\
&= \gamma (2 \mu d - \norm{v'_t}_2^2) \D{t} + \sqrt{2 \gamma \mu}(v'_t \cdot \D{B_t}).
\end{align*}

So, we can write the Hamiltonian at any time as a function of the initial Hamiltonian $\phi_0$ and the random variables $B_t$ and $v'_t$ as:

$$\phi_t = \phi_0 - \gamma \int_0^t \norm{v'_s}_2^2 \D{s} + \sqrt{2 \gamma \mu}\int_0^t \norm{v'_s}_2 \frac{v'_s}{\norm{v'_s}_2}  \cdot \D{B_s} + 2 \gamma \mu d t.$$

Let $V_t$ denote $\int_0^t \norm{v'_s}_2^2 \D{s}$. By scalability of Brownian motion, we can define a Brownian motion $B'_t$ jointly distributed with $B_t$ such that $\D{B_t} = \frac{1}{\norm{v'_t}_2}  \frac{\D{}}{\D{t}}\int_{0}^{V_t} \D{B'_s}$. Then, we have:

$$\phi_t = \phi_0 - \gamma V_t + \sqrt{2 \gamma \mu}\int_0^{V_t} \frac{v'_{g(s)}}{\norm{v'_{g(s)}}_2} \cdot \D{B'_s} + 2 \gamma \mu dt,$$

Where $g(r)$ is the value $r'$ such that $\int_0^{r'} \norm{v'_s}_2^2 \D{s} = r$. We can then use the rotational symmetry of Brownian motion to define another Brownian motion $B''_t$ jointly distributed with $B'_t$ such that $u \cdot \D{B''_t} =  \frac{v'_{g(t)}}{\norm{v'_{g(t)}}_2} \cdot \D{B'_t}$ for a fixed unit vector $u$, giving:

$$\phi_t = \phi_0 - \gamma V_t + \sqrt{2 \gamma \mu}\int_0^{V_t} u \cdot \D{B''_s} + 2 \gamma \mu dt.$$

We will show that with probability at least $1 - \delta$ over $B''_t$, the maximum of $\phi'(V) := \phi_0 - \gamma V + \sqrt{2 \gamma \mu} \int_0^V u \cdot \D{B''_s}$ over $V \in [0, \infty)$ is at most $\frac{1}{4} v_{\max}^2$. Under this event, if $c$ is sufficiently large then for all $t \in [0, \tau)$ we have $\phi_t \leq \frac{1}{4}v_{\max}^2 + 2 \gamma \mu d \tau \leq \frac{1}{2} v_{\max}^2$, giving the desired velocity bound.

We first claim that with probability at at least $1 - \frac{\delta}{2}$. for all non-negative integers $k$, we have $\phi'(k v_{\max}^2) \leq - \frac{(k-1) v_{\max}^2}{2}$. For sufficiently large $c$, this holds for $k = 0$ with probability at least $1 - \frac{\delta}{4}$ by \eqref{eq:startingbound-ud}. Conditioning on this event, for $k > 0$ if $\phi'(k v_{\max}^2) \geq - \frac{(k-1) v_{\max}^2}{2}$, then:

$$
\sqrt{2\gamma \mu} \int_0^{k v_{\max}^2} u \cdot \D{B''_s} = N(0, 2 k \gamma \mu v_{\max}^2) \geq - \frac{(k-1) v_{\max}^2}{2} - \phi_0 + k \gamma v_{\max}^2 
\geq (\gamma - 1) k v_{\max}^2,
$$

Which occurs with probability at most $\exp(-\frac{(\gamma - 1)^2 k^2 v_{\max}^4}{4k\gamma \mu v_{\max}^2}) \leq \exp(-\frac{k v_{\max}^2}{8 \mu}) $. If the constant $c$ in $v_{\max}$ is sufficiently large, then this is less than $\frac{\delta^{k+2}}{2}$. Taking a union bound over all $k$, we get the claim. Next, we claim that in each interval $[k v_{\max}^2, (k+1)v_{\max}^2)$, the maximum increase of $\phi'(V)$ is more than $(\frac{k+1}{2})v_{\max}^2$ with probability at most $\frac{\delta^{k+2}}{2}$. Taking a union bound over all intervals, this claim along with the previous claim this gives the desired bound on $\phi'(V)$ with probability $1 - \delta$. This claim follows by observing that in the interval $[k v_{\max}^2, (k+1)v_{\max}^2)$, $\phi'(V)$ increases more than $\max_{V \in [k v_{\max}^2, (k+1)v_{\max}^2)} \left[\int_{k v_{\max}^2}^V u \cdot \D{B''_s}\right]$, which is at most $(\frac{k+1}{2})v_{\max}^2$ with probability at most $\exp(- \frac{(\frac{k+1}{2})^2 v_{\max}^4}{8 v_{\max}^2}) \leq \frac{\delta^{k+1}}{2}$.

The discrete chain is analyzed similarly. We have:
\begin{align*}
\D{\phi_t} &= \frac{\del \phi_t}{\del x_t} \cdot \D{x_t} + \frac{\del \phi_t}{\del v_t} \cdot \D{v_t} + \frac{1}{2} \left[ \sum_{i, j \in [d]} \frac{\del^2 \phi_t}{\D{(v_t)_i} \D{(v_t)_j}} \frac{\D{(v_t)_i}}{\D{B_t}} \frac{\D{(v_t)_j}}{\D{B_t}} \right] \D{t}\\
&= \mu \grad f(x_t) \cdot v_t \D{t} + v_t \cdot (- \mu \grad f(x_{\lfloor\frac{t}{\eta}\rfloor\eta}) \D{t} - \gamma v_t \D{t} + \sqrt{2 \gamma \mu} \D{B_t}) + 2 \gamma \mu d \cdot \D{t}\\ 
&=\mu (\grad f(x_t) - \grad f(x_0)) \cdot v_t \D{t} - \gamma \norm{v_t}_2^2 \D{t} + \sqrt{2 \gamma \mu}(v \cdot \D{B_t}) + 2 \gamma \mu d \cdot \D{t}\\ 
&\leq \mu L \norm{x_t - x_{\lfloor\frac{t}{\eta}\rfloor\eta}}_2 \norm{v_t}_2 \D{t} - \gamma \norm{v_t}_2^2 \D{t} + \sqrt{2 \gamma \mu}(v \cdot \D{B_t}) + 2 \gamma \mu d \cdot \D{t} \\
&= \mu L \norm{\int_{\lfloor\frac{t}{\eta}\rfloor\eta}^t v_s \D{s}}_2 \norm{v_t}_2 \D{t} - \gamma \norm{v_t}_2^2 \D{t} + \sqrt{2 \gamma \mu}(v \cdot \D{B_t}) + 2 \gamma \mu d \cdot \D{t}\\
&\leq \mu L \left( \int_{\lfloor\frac{t}{\eta}\rfloor\eta}^t \norm{v_s}_2 \norm{v_t}_2 \D{s} \right) \D{t} - \gamma \norm{v_t}_2^2 \D{t} + \sqrt{2 \gamma \mu}(v \cdot \D{B_t}) + 2 \gamma \mu d \cdot \D{t}\\
&\leq \frac{\mu L}{2} \left( \int_{\lfloor\frac{t}{\eta}\rfloor\eta}^t \norm{v_s}_2^2 + \norm{v_t}_2^2 \D{s} \right) \D{t} - \gamma \norm{v_t}_2^2 \D{t} + \sqrt{2 \gamma \mu}(v \cdot \D{B_t}) + 2 \gamma \mu d \cdot \D{t}.
\end{align*}

Integrating, we get:

\begin{align*}
\phi_t &\leq \phi_0 - (\gamma - \frac{\mu L \eta}{2}) \int_0^t \norm{v_s}_2^2 \D{s} + \sqrt{2 \gamma \mu}\int_0^t \norm{v_s}_2 \frac{v_s}{\norm{v_s}_2}  \cdot \D{B_s} + 2 \gamma \mu d t\\
&\leq \phi_0 - \frac{\gamma}{2} \int_0^t \norm{v_s}_2^2 \D{s} + \sqrt{2 \gamma \mu}\int_0^t \norm{v_s}_2 \frac{v_s}{\norm{v_s}_2}  \cdot \D{B_s} + 2 \gamma \mu d t.
\end{align*}

At this point we repeat the analysis from the continuous case (only losing a multiplicative constant due to the $\gamma / 2$ multiplier not being $\gamma$).
\end{proof}
\section{Discussion and Open Questions}
\label{section:discussion}
Our work raises several interesting questions. While our bounds are for log-smooth and strongly log-concave distributions, it would be interesting to relax these assumptions. The known results for the continuous process in the underdamped case are only for weaker measures, and it is compelling to extend them to \Renyi divergence. 
Our result has a seemingly curious property: the finite time behaviour of the discrete chain is shown to be close in \Renyi divergence to the target distribution, yet we do not know if the stationary distribution of the discrete chain satisfies this property. %This is reminiscent of some results in privacy literature, where e.g., running noisy SGD for a fixed number of steps can be shown to be private but we know nothing of the privacy if we run noisy SGD for many more steps~\citep{BassilyST,FeldmanMTT18}. 
Addressing this gap in our understanding is left to future work. 
There are several variants of these methods that have been studied (e.g. Metropolis Adjusted Langevin Algorithm, Hamiltonian Monte Carlo, Stochastic Gradient Langevin Dynamics) and extending our techniques to these methods would be interesting. Finally, applying these tools to specific non-convex functions of interest such as the Rayleigh quotient may lead to more practical efficient algorithms for problems such as private PCA~\citep{KapralovT}.

We note that our bound on iteration complexity for the overdamped Langevin dynamics are proportional to $\tilde{O}(1/\eps^2)$, as opposed to e.g. a $O(1/\eps^{1/2})$ dependence in \cite{MFWB19} for KL-divergence. In many differential privacy applications we would set $\eps$ to be not too small a constant, so this gap may be acceptable from a practical standpoint. Obtaining better dependencies on $\eps$ remains an interesting question. We believe the loss of a $1/\eps^2$ factor in our ``unconditioning'' argument is unavoidable, and so alternate analyses may be needed to improve this dependence.

\bibliographystyle{plainnat}
\bibliography{ref}
\end{document}